\documentclass[11pt]{article}
\usepackage{fullpage,times}

\usepackage{hyperref}
\usepackage{url}

\usepackage{enumitem}

\usepackage{amsthm,amsfonts,amsmath,amssymb,bbm,epsfig,color,float,graphicx,verbatim}
\usepackage{algorithm,algorithmic}

\newtheorem{theorem}{Theorem}

\newtheorem{lemma}{Lemma}
\newtheorem{corollary}{Corollary}
\newtheorem{definition}{Definition}
\newtheorem{assumption}{Assumption}

\newcommand{\reals}{\mathbb{R}}
\newcommand{\bN}{\mathbb{N}}
\newcommand{\E}{\mathbb{E}}

\newcommand{\circpar}[1]{\left(#1\right)}

\newcommand{\be}{\mathbf{e}}

\newcommand{\bw}{\mathbf{w}}
\newcommand{\bzeta}{\boldsymbol\zeta}
\newcommand{\bg}{\mathbf{g}}
\newcommand{\bb}{\mathbf{b}}

\newcommand{\bz}{\mathbf{z}}

\newcommand{\bsigma}{\boldsymbol{\sigma}}

\newcommand{\Ocal}{\mathcal{O}}

\newcommand{\Wcal}{\mathcal{W}}

\newcommand{\norm}[1]{\left\|#1\right\|}
\newcommand{\inner}[1]{\langle#1\rangle}

\newcommand{\secref}[1]{Sec.~\ref{#1}}

\renewcommand{\eqref}[1]{Eq.~(\ref{#1})}
\newcommand{\lemref}[1]{Lemma~\ref{#1}}

\newcommand{\thmref}[1]{Thm.~\ref{#1}}

\newcommand{\algspan}{\operatorname{span}}
\newcommand{\absval}[1]{\left|#1\right|}
\newcommand{\eqdef}{\mathrel{\mathop:}=}
\DeclareMathOperator{\conv}{conv}

\numberwithin{fact_counter}{part}

\title{Communication Complexity of Distributed\\ Convex Learning and Optimization}

\author{
Yossi Arjevani \\
Weizmann Institute of Science\\
Rehovot 7610001, Israel \\
\texttt{yossi.arjevani@weizmann.ac.il} \\
\and
Ohad Shamir\\
Weizmann Institute of Science\\
Rehovot 7610001, Israel \\
\texttt{ohad.shamir@weizmann.ac.il} \\
}	
\date{}
\begin{document}

\maketitle

\begin{abstract}
We study the fundamental limits to communication-efficient distributed
methods for convex learning and optimization, under different assumptions
on the information available to individual machines, and the types of
functions considered. We identify cases where existing algorithms are
already worst-case optimal, as well as cases where room for
further improvement is still possible. Among other things, our results
indicate that without similarity between the local objective functions (due to statistical data similarity or otherwise) many communication rounds may be required, even if the machines have unbounded computational power.
\end{abstract}

\section{Introduction}

We consider the problem of distributed convex learning and optimization,
where a set of $m$ machines, each with access to a different local convex
function $F_i:\reals^d\mapsto \reals$ and a convex domain $\Wcal\subseteq
\reals^d$, attempt to solve the optimization problem
\begin{equation}\label{eq:kavg}
\min_{\bw\in\Wcal} F(\bw)~~~\text{where}~~~ F(\bw)=\frac{1}{m}\sum_{i=1}^{m}F_i(\bw).
\end{equation}
A prominent application is empirical risk minimization, where the goal is to
minimize the average loss over some dataset, where each machine has access to
a different subset of the data. Letting $\{\bz_1,\ldots,\bz_N\}$ be the dataset
composed of $N$ examples, and assuming the loss function $\ell(\bw,\bz)$ is
convex in $\bw$, then the empirical risk minimization problem
$\min_{\bw\in\Wcal} \frac{1}{N}\sum_{i=1}^{N} \ell(\bw,\bz_i)$ can be written
as in \eqref{eq:kavg}, where $F_i(\bw)$ is the average loss over machine $i$'s examples.

The main challenge in solving such problems is that communication between the
different machines is usually slow and constrained, at least compared to the
speed of local processing. On the other hand, the datasets
involved in distributed learning are usually large and high-dimensional. Therefore, machines cannot
simply communicate their entire data to each other, and the question is how
well can we solve problems such as \eqref{eq:kavg} using as little
communication as possible.

As datasets continue to increase in size, and parallel computing platforms
becoming more and more common (from multiple cores on a single CPU to
large-scale and geographically distributed computing grids), distributed
learning and optimization methods have been the focus of
much research in recent years, with just a few examples including
\cite{zhang2013communication,bekkerman2011scaling,balcandistributed,ZWSL11,AgChDuLa11,BoPaChPeEc11,MaKeSyBo13,Ya13,ReReWrNi11,RiTa13,DeGiShXi12,CoShSrSr11,DuAgWa12,jaggi2014communication,shamir2014communication,shamir2014distributed,balcan2014improved,zhang2015communication}.
Most of this work studied algorithms for this problem, which provide upper
bounds on the required time and communication complexity.

In this paper, we take the opposite direction, and study what are the
fundamental performance limitations in solving \eqref{eq:kavg}, under several
different sets of assumptions. We identify cases where existing algorithms
are already optimal (at least in the worst-case), as well as cases where
room for further improvement is still possible.

Since a major constraint in distributed learning is communication, we focus
on studying the amount of communication required to optimize \eqref{eq:kavg}
up to some desired accuracy $\epsilon$. More precisely, we consider the
number of \emph{communication rounds} that are required, where in each
communication round the machines can generally broadcast to each other
information linear in the problem's dimension $d$ (e.g. a point in $\Wcal$ or
a gradient). This applies to virtually all algorithms for large-scale
learning we are aware of, where sending vectors and gradients is feasible, but computing and
sending larger objects, such as Hessians ($d\times d$ matrices) is not.

Our results pertain to several possible settings (see
\secref{sec:preliminaries} for precise definitions). First, we distinguish
between the local functions being merely convex or strongly-convex, and
whether they are smooth or not. These distinctions are standard in studying
optimization algorithms for learning, and capture important properties such
as the regularization and the type of loss function used. Second, we
distinguish between a setting where the local functions are related -- e.g.,
because they reflect statistical similarities in the data residing at different machines -- and a setting where no relationship is assumed. 
For example, in the extreme case where data was split uniformly at random between machines, one can show that quantities such as the values, gradients
and Hessians of the local functions differ only by $\delta=\Ocal(1/\sqrt{n})$,
where $n$ is the sample size per machine, due to concentration of measure effects. Such similarities can be used to speed up
the optimization/learning process, as was done in e.g.
\cite{shamir2014communication,zhang2015communication}. Both the $\delta$-related and the unrelated
setting can be considered in a unified way, by letting $\delta$ be a parameter
and studying the attainable lower bounds as a function of $\delta$.
Our results can be summarized as follows:
\begin{itemize}[leftmargin=*]
  \item First, we define a mild structural assumption on the algorithm
      (which is satisfied by reasonable approaches we are aware of), which
      allows us to provide the lower bounds described below on the number
      of communication rounds required to reach a given suboptimality
      $\epsilon$.
  \begin{itemize}[leftmargin=*]
    \item When the local functions can be unrelated, we prove a lower
        bound of $\Omega(\sqrt{1/\lambda}\log(1/\epsilon))$ for smooth
        and $\lambda$-strongly convex functions, and
        $\Omega(\sqrt{1/\epsilon})$ for smooth convex functions. These
        lower bounds are matched by a straightforward distributed
        implementation of accelerated gradient descent. In particular,
        the results imply that many communication rounds may be required
        to get a high-accuracy solution, and moreover, that no algorithm
        satisfying our structural assumption would be better, even if we
        endow the local machines with unbounded computational power. For
        non-smooth functions, we show a lower bound of
        $\Omega(\sqrt{1/\lambda \epsilon})$ for $\lambda$-strongly convex
        functions, and $\Omega(1/\epsilon)$ for general convex functions. Although
        we leave a full derivation to future work, it seems these lower bounds can be matched in our framework
        by an algorithm combining acceleration and Moreau proximal smoothing of the local functions.
    \item When the local functions are related (as quantified by the parameter $\delta$), we prove a communication round lower bound of $\Omega(\sqrt{\delta/\lambda}\log(1/\epsilon))$ for smooth and $\lambda$-strongly convex functions. For quadratics, this bound is matched by (up to constants and logarithmic factors)
        by the recently-proposed DISCO algorithm
        \cite{zhang2015communication}. However, getting an optimal algorithm for general strongly convex and smooth functions in the $\delta$-related setting, let alone for non-smooth or non-strongly convex functions, remains open.
  \end{itemize}
  \item We also study the attainable performance without posing any
      structural assumptions on the algorithm, but in the more restricted
      case where only a single round of communication is allowed. We
      prove that in a broad regime, the performance of any distributed algorithm may be no better than a `trivial' algorithm which returns the minimizer of one of the local functions, as long as the number of bits communicated is less than $\Omega(d^2)$.
      Therefore, in our setting, no communication-efficient 1-round distributed algorithm can provide non-trivial performance in the worst case.
\end{itemize}

\subsection*{Related Work}

There have been several previous works which considered lower bounds in the
context of distributed learning and optimization, but to the best of our
knowledge, none of them provide a similar type of results. Perhaps the most
closely-related paper is \cite{tsilu87}, which studied the communication
complexity of distributed optimization, and showed that
$\Omega(d\log(1/\epsilon))$ bits of communication are necessary between the
machines, for $d$-dimensional convex problems. However, in our setting this
does not lead to any non-trivial lower bound on the number of communication
rounds (indeed, just specifying a $d$-dimensional vector up to accuracy
$\epsilon$ required $\Ocal(d\log(1/\epsilon))$ bits). More recently,
\cite{balcandistributed} considered lower bounds for certain types of
distributed learning problems, but not convex ones in an agnostic
distribution-free framework. In the context of lower bounds for one-round
algorithms, the results of \cite{clwo09} imply that $\Omega(d^2)$ bits of
communication are required to solve linear regression in one round of
communication. However, that paper assumes a different model than ours, where
the function to be optimized is not split among the machines as in
\eqref{eq:kavg}, where each $F_i$ is convex. Moreover, issues such as strong
convexity and smoothness are not considered. \cite{shamir2014communication}
proves an impossibility result for a one-round distributed learning scheme,
even when the local functions are not merely related, but actually result
from splitting data uniformly at random between machines. On the flip side,
that result is for a particular algorithm, and doesn't apply to any possible
method.

Finally, we emphasize that distributed learning and optimization can be
studied under many settings, including ones different than those studied
here. For example, one can consider distributed learning on a
stream of i.i.d. data
\cite{shamir2014distributed,CoShSrSr11,frostig2014competing,DeGiShXi12}, or
settings where the computing architecture is different, e.g. where the
machines have a shared memory, or the function to be optimized is not split as in \eqref{eq:kavg}. Studying lower bounds in
such settings is an interesting topic for future work.

\section{Notation and Framework}\label{sec:preliminaries}

The only vector and matrix norms used in this paper are the Euclidean norm and
the spectral norm, respectively. $\be_j$ denotes the $j$-th standard unit vector. We let $\nabla G(\bw)$ and $\nabla^2 G(\bw)$
denote the gradient and Hessians of a function $G$ at $\bw$, if they exist. $G$ is smooth (with
parameter $L$) if it is differentiable and the gradient is
$L$-Lipschitz. In particular, if $\bw^*=\arg\min_{\bw\in\Wcal}G(\bw)$, then $G(\bw)-G(\bw^*)\leq \frac{L}{2}\norm{\bw-\bw^*}^2$. $G$ is strongly convex (with parameter $\lambda$) if for any
$\bw,\bw'\in \Wcal$, $G(\bw')\geq
G(\bw)+\inner{\bg,\bw'-\bw}+\frac{\lambda}{2}\norm{\bw'-\bw}^2$ where $\bg\in
\partial G(\bw')$ is a subgradient of $G$ at $\bw$. In particular, if
$\bw^*=\arg\min_{\bw\in\Wcal}G(\bw)$, then $G(\bw)-G(\bw^*)\geq
\frac{\lambda}{2}\norm{\bw-\bw^*}^2$. Any convex function is also
strongly-convex with $\lambda=0$. A special case of smooth convex functions
are quadratics, where $G(\bw)=\bw^\top A \bw+\mathbf{b}^\top \bw+c$ for some
positive semidefinite matrix $A$, vector $\mathbf{b}$ and scalar $c$. In this
case, $\lambda$ and $L$ correspond to the smallest and largest eigenvalues of
$A$.

We model the distributed learning algorithm as an iterative process, where in each round the machines may perform some local computations, followed by a
communication round where each machine broadcasts a message to all other
machines. We make no assumptions on the computational complexity of the local computations. After all communication rounds are completed, a designated
machine provides the algorithm's output (possibly after additional local
computation).

Clearly, without any assumptions on the number of bits communicated, the
problem can be trivially solved in one round of communication (e.g. each
machine communicates the function $F_i$ to the designated machine, which then
solves \eqref{eq:kavg}. However, in
practical large-scale scenarios, this is non-feasible, and the size of each
message (measured by the number of bits) is typically on the order of
$\tilde{\Ocal}(d)$, enough to send a $d$-dimensional real-valued vector\footnote{The
$\tilde{\Ocal}$ hides constants and factors logarithmic in the required
accuracy of the solution. The idea is that we can represent real numbers up
to some arbitrarily high machine precision, enough so that finite-precision
issues are not a problem.}, such as points in the optimization domain or
gradients, but not larger objects such as $d\times d$ Hessians.

In this model, our main question is the following: How many rounds of
communication are \emph{necessary} in order to solve problems such as
\eqref{eq:kavg} to some given accuracy $\epsilon$?

As discussed in the introduction, we first need to distinguish between
different assumptions on the possible relation between the local functions.
One natural situation is when no significant relationship can be assumed, for instance
when the data is arbitrarily split or is gathered by each machine from 
statistically dissimilar sources. We denote this as the
\emph{unrelated} setting. However, this assumption is often unnecessarily
pessimistic. Often the data allocation process is more random, or we can assume that the different data sources for each machine have statistical similarities (to give a simple example,
consider learning from users' activity across a geographically distributed computing grid, each servicing its own local population).
We will capture such similarities, in the context of quadratic functions, using the following definition:
\begin{definition}\label{def:related}
We say that a set of quadratic functions
\begin{align*}
	F_i(\bw)&\eqdef\bw^\top A_i \bw + \bb_i \bw + c_i,\qquad A_i\in\reals^{d\times d},~ \bb_i\in\reals^d,~ c_i\in\reals
\end{align*}
are \emph{$\delta$-related}, if for any $i,j\in \{1\ldots k\}$, it holds that
\begin{align*}
	\norm{A_i-A_j} \le \delta,~\norm{\bb_i-\bb_j}\le \delta,~\absval{c_i-c_j}\le \delta
\end{align*}
\end{definition}
For example, in the context of linear regression with the squared loss over a
bounded subset of $\reals^d$, and assuming $mn$ data points with bounded norm
are randomly and equally split among $m$ machines, it can be shown that the
conditions above hold with $\delta=\Ocal(1/\sqrt{n})$
\cite{shamir2014communication}. The choice of $\delta$ provides us with a
spectrum of learning problems ranked by difficulty: When $\delta=\Omega(1)$,
this generally corresponds to the unrelated setting discussed earlier.
When $\delta=\Ocal(1/\sqrt{n})$, we get the situation typical of
randomly partitioned data. When $\delta=0$, then all the local functions have
essentially the same minimizers, in which case \eqref{eq:kavg} can be
trivially solved with zero communication, just by letting one machine
optimize its own local function. We note that although Definition
\ref{def:related} can be generalized to non-quadratic functions, we do not
need it for the results presented here.

We end this section with an important remark. In this paper, we prove lower
bounds for the $\delta$-related setting, which includes as a special case the commonly-studied setting of randomly partitioned data (in which case $\delta=\Ocal(1/\sqrt{n})$). However, our bounds \emph{do not apply} for random partitioning, since they use $\delta$-related constructions which do not correspond to randomly partitioned data. In fact, very recent work \cite{lee2015distributed} has cleverly shown that for randomly partitioned data, and for certain reasonable regimes of strong convexity and smoothness, it is actually possible to get better performance than what is indicated by our lower bounds. However, this encouraging result crucially relies on the random partition property, and in parameter regimes which limit how much each data point needs to be ``touched'', hence preserving key statistical independence properties. We suspect that it may be difficult to improve on our lower bounds under substantially weaker assumptions.

%


\section{Lower Bounds Using a Structural
Assumption}\label{sec:multiple}

In this section, we present lower bounds on the number of communication
rounds, where we impose a certain mild structural assumption on the
operations performed by the algorithm. Roughly speaking, our lower bounds
pertain to a very large class of algorithms, which are based on linear
operations involving points, gradients, and vector products with local
Hessians and their inverses, as well as solving local optimization problems
involving such quantities. At each communication round, the machines can
share any of the vectors they have computed so far. Formally, we consider
algorithms which satisfy the assumption stated below. For convenience, we
state it for smooth functions (which are differentiable) and discuss the case
of non-smooth functions in \secref{sec:nonsmooth}.
\begin{assumption} \label{assump:dyn}
For each machine $j$, define a set $W_j\subset \reals^d$, initially
$W_j=\{\mathbf{0}\}$. Between communication rounds, each machine $j$
iteratively computes and adds to $W_j$ some finite number of points $\bw$,
each satisfying
\begin{align}
\gamma\bw &+ \nu \nabla F_j(\bw) \in  \algspan\Big\{\bw'~,~\nabla F_j(\bw')~,~(\nabla^2 F_j(\bw')+D)\bw''~,~(\nabla^2 F_j(\bw')+D)^{-1}\bw'' ~\Big|~\nonumber\\
&
	\bw',\bw'' \in W_j~~,~~ D~\text{diagonal}~~,~~\nabla^2 F_j(\bw')~\text{exists}~~,~~(\nabla^2 F_j(\bw')+D)^{-1}~\text{exists}\Big\}.
	\label{eq:assumpdyn}
\end{align}
for some $\gamma,\nu\geq 0$ such that $\gamma+\nu> 0$. After every communication round, let $W_j := \cup_{i=1}^{m}W_i$ for all $j$.
The algorithm's final output (provided by the designated machine $j$) is a
point in the span of $W_j$.
%
%
%
\end{assumption}
This assumption requires several remarks:
\begin{itemize}[leftmargin=*]
\item Note that $W_j$ is not an explicit part of the algorithm: It simply
    includes all points computed by machine $j$ so far, or communicated to
    it by other machines, and is used to define the set of new points which
    the machine is allowed to compute.
\item The assumption bears some resemblance -- but is far weaker -- than
    standard assumptions used to provide lower bounds for iterative
    optimization algorithms. For example, a common assumption (see
    \cite{nesterov2004introductory}) is that each computed point $\bw$ must
    lie in the span of the previous gradients. This corresponds to a
    special case of Assumption \ref{assump:dyn}, where $\gamma=1,\nu=0$,
    and the span is only over gradients of previously computed points. Moreover, it also allows (for
    instance) exact optimization of each local function, which is a
    subroutine in some distributed algorithms (e.g.
    \cite{ZWSL11,zhang2013communication}), by setting $\gamma=0,\nu=1$ and
    computing a point $\bw$ satisfying $\gamma \bw+\nu \nabla
    F_j(\bw)=\mathbf{0}$. By allowing the span to include previous
    gradients, we also incorporate algorithms which perform optimization
    of the local function plus terms involving previous gradients and
    points, such as \cite{shamir2014communication}, as well as algorithms
    which rely on local Hessian information and preconditioning, such as
    \cite{zhang2015communication}. In summary, the assumption is satisfied by most techniques for black-box convex optimization that we are aware of. Finally, we emphasize that we do not restrict the number or computational complexity of the operations performed between communication rounds.
\item The requirement that $\gamma,\nu\geq 0$ is to exclude algorithms
    which solve non-convex local optimization problems of the form
    $\min_{\bw} F_j(\bw)+\gamma\norm{\bw}^2$ with $\gamma<0$, which
    are unreasonable in practice and can sometimes break our lower bounds.
\item The assumption that $W_j$ is initially $\{\mathbf{0}\}$ (namely, that the
    algorithm starts from the origin) is purely for convenience, and our
    results can be easily adapted to any other starting point by shifting
    all functions accordingly.
\end{itemize}

The techniques we employ in this section are
inspired by lower bounds on the iteration complexity of first-order methods for standard
(non-distributed) optimization (see for example \cite{nesterov2004introductory}). These are based on the construction of `hard'
functions, where each gradient (or subgradient) computation can only provide
a small improvement in the objective value. In our setting, the dynamics are
roughly similar, but the necessity of many gradient computations is replaced
by many communication rounds. This is achieved by constructing suitable local
functions, where at any time point no individual machine can `progress' on
its own, without information from other machines.

\subsection{Smooth Local Functions}\label{sec:smooth}

We begin by presenting a lower bound when the local functions $F_i$ are
strongly-convex and smooth:

\begin{theorem} \label{thm:smooth_lb}
For any even number $m$ of machines, any distributed algorithm which
satisfies Assumption \ref{assump:dyn}, and for any $\lambda\in[0,1),
\delta\in(0,1)$, there exist $m$ local quadratic functions over $\reals^d$
(where $d$ is sufficiently large) which are $1$-smooth, $\lambda$-strongly
convex, and $\delta$-related, such that if $\bw^*=\arg\min_{\bw\in \reals^d} F(\bw)$, then the number of
communication rounds required to obtain $\hat{\bw}$ satisfying
$F(\hat{\bw})-F(\bw^*)\leq \epsilon$ (for any $\epsilon>0$) is at
least
\begin{align*}
    \frac{1}{4}\left(\sqrt{1+\delta\left(\frac{1}{\lambda}-1\right)}-1\right)\log\left(\frac{\lambda\norm{\bw^*}^2}{4\epsilon}\right)-\frac{1}{2}
    ~=~ \Omega\left(\sqrt{\frac{\delta}{\lambda}}\log\left(\frac{\lambda\norm{\bw^*}^2}{\epsilon}\right)\right)
\end{align*}
if $\lambda>0$, and at least $\sqrt{\frac{3\delta}{32 \epsilon }
}\norm{\bw^*}-2$ if $\lambda=0$.
\end{theorem}
The assumption of $m$ being even is purely for technical convenience, and can
be discarded at the cost of making the proof slightly more complex. Also,
note that $m$ does not appear explicitly in the bound, but may appear
implicitly, via $\delta$ (for example, in a statistical setting $\delta$ may depend on the number of data points per machine, and may be larger if the same dataset is divided to more machines). 

Let us contrast our lower bound with some existing algorithms and guarantees
in the literature. First, regardless of whether the local functions are
similar or not, we can always simulate any gradient-based method designed for
a single machine, by iteratively computing gradients of the local functions,
and performing a communication round to compute their average. Clearly, this
will be a gradient of the objective function
$F(\cdot)=\frac{1}{m}\sum_{i=1}^{m}F_i(\cdot)$, which can be fed into any
gradient-based method such as gradient descent or accelerated gradient
descent \cite{nesterov2004introductory}. The resulting number of required
communication rounds is then equal to the number of iterations. In
particular, using accelerated gradient descent for smooth and
$\lambda$-strongly convex functions yields a round complexity of
$\Ocal(\sqrt{1/\lambda}\log(\norm{\bw^*}^2/\epsilon))$, and
$\Ocal(\norm{\bw^*}\sqrt{1/\epsilon})$ for smooth convex functions. This
matches our lower bound (up to constants and log factors) when the local
functions are unrelated ($\delta=\Omega(1)$).

When the functions are related, however, the upper bounds above are highly
sub-optimal: Even if the local functions are completely identical, and
$\delta=0$, the number of communication rounds will remain the same as when
$\delta=\Omega(1)$. To utilize function similarity while guaranteeing arbitrary small $\epsilon$, the two most relevant
algorithms are DANE \cite{shamir2014communication}, and the more recent DISCO
\cite{zhang2015communication}. For smooth and $\lambda$-strongly convex
functions, which are either quadratic or satisfy a certain self-concordance
condition, DISCO achieves $\tilde{\Ocal}(1+\sqrt{\delta/\lambda})$
round complexity (\cite[Thm.2]{zhang2015communication}), which matches our
lower bound in terms of dependence on $\delta,\lambda$. 
However, for non-quadratic losses, the round complexity bounds are somewhat worse, and there are no guarantees for strongly convex and smooth functions which are not self-concordant. Thus, the question of the optimal round complexity for such functions remains open.

The full proof of \thmref{thm:smooth_lb} appears in the supplementary
material, and is based on the following idea: For simplicity, suppose we have
two machines, with local functions $F_1,F_2$ defined as follows,
\begin{gather}
	F_1(\bw)= \frac{\delta(1-\lambda)}{4}\bw^\top A_1 \bw-\frac{\delta(1-\lambda)}{2}\be_1^\top \bw + \frac{\lambda}{2}\norm{\bw}^2 \label{eq:fdef}\\
   F_2(\bw)= \frac{\delta(1-\lambda)}{4}\bw^\top A_2\bw + \frac{\lambda}{2}\norm{\bw}^2,~~~\text{where}\notag\\
		A_1 = \left[~
  \begin{matrix}
  1 & 0 & 0 & 0 & 0 & 0 &  \hdots \\
  0 & 1 & -1 & 0 & 0 & 0 &  \hdots \\
  0 & -1 & 1 & 0 & 0 & 0 &  \hdots \\
  0 & 0 & 0 & 1 & -1 & 0 &  \hdots \\
  0 & 0 & 0 & -1  & 1 & 0 &  \hdots \\
  \vdots & \vdots & \vdots & \vdots & \vdots & \vdots & \vdots
  \end{matrix}
  ~\right],~~~
		A_2 = \left[~
  \begin{matrix}
  1 & -1 & 0 & 0 & 0 & 0 & \hdots \\
  -1 & 1 & 0 & 0 & 0 & 0 & \hdots \\
  0 & 0 & 1 & -1 & 0 & 0 & \hdots \\
  0 & 0 & -1 & 1 & 0 & 0 & \hdots \\
  0 & 0 & 0 & 0 & 1 & -1 & \hdots \\
  0 & 0 & 0 & 0 & -1 & 1 & \hdots \\
  \vdots & \vdots & \vdots & \vdots & \vdots & \vdots &  \vdots
  \end{matrix}
  ~\right]\notag
\end{gather}

It is easy to verify that for $\delta,\lambda\leq 1$, both $F_1(\bw)$ and
$F_2(\bw)$ are $1$-smooth and $\lambda$-strongly convex, as well as
$\delta$-related. Moreover, the optimum of their average is a point $\bw^*$
with non-zero entries at all coordinates. However, since each local functions
has a block-diagonal quadratic term, it can be shown that for any algorithm
satisfying Assumption \ref{assump:dyn}, after $T$ communication rounds, the
points computed by the two machines can only have the first $T+1$ coordinates non-zero. No machine will be able to further `progress' on its
own, and cause additional coordinates to become non-zero, without another
communication round. This leads to a lower bound on the optimization error
which depends on $T$, resulting in the theorem statement after a few
computations.

\subsection{Non-smooth Local Functions}\label{sec:nonsmooth}

Remaining in the framework of algorithms satisfying Assumption
\ref{assump:dyn}, we now turn to discuss the situation where the local
functions are not necessarily smooth or differentiable. For simplicity, our formal results here will be in the unrelated setting, and we only informally discuss their extension to a $\delta$-related setting (in a sense relevant to non-smooth functions). Formally defining $\delta$-related non-smooth functions is possible but not altogether trivial, and is therefore left to future work.

We adapt Assumption \ref{assump:dyn} to the non-smooth case,
by allowing gradients to be replaced by arbitrary subgradients at the same points. Namely, we replace \eqref{eq:assumpdyn} by 
the requirement that for some $\bg\in \partial F_j(\bw)$, and $\gamma,\nu\geq 0, \gamma+\nu>0$,
\begin{align*}
\gamma\bw &+ \nu \bg \in  \algspan\Big\{\bw'~,~\bg'~,~(\nabla^2 F_{j}(\bw')+D)\bw''~,~(\nabla^2 F_{j}(\bw')+D)^{-1}\bw'' ~\Big|~\\
&
	\bw',\bw'' \in W_j~,~ \bg'\in \partial F_{j}(\bw')~,~D~\text{diagonal}~,~\nabla^2 F_{j}(\bw')~\text{exists}~,~(\nabla^2 F_{j}(\bw')+D)^{-1}~\text{exists}\Big\}.
\end{align*}


The lower bound for this setting is stated in the following theorem.
\begin{theorem} \label{thm:nonsmooth_lb}
For any even number $m$ of machines, any distributed optimization algorithm
which satisfies Assumption \ref{assump:dyn}, and for any $\lambda\geq 0$,
there exist $\lambda$-strongly convex $(1+\lambda)$-Lipschitz continuous convex local functions
$F_1(\bw)$ and $F_2(\bw)$ over the unit Euclidean ball in $\reals^d$ (where
$d$ is sufficiently large), such that if $\bw^*=\arg\min_{\bw:\norm{\bw}\leq 1} F(\bw)$, the number of communication rounds
required to obtain $\hat{\bw}$ satisfying $F(\hat{\bw})-F(\bw^*)\leq
\epsilon$ (for any sufficiently small $\epsilon>0$) is
$\frac{1}{8\epsilon}-2$ for $\lambda=0$, and
      $\sqrt{\frac{1}{16\lambda\epsilon}}-2$ for $\lambda>0$.
\end{theorem}
As in \thmref{thm:smooth_lb}, we note that the assumption of even $m$ is for
technical convenience.

This theorem, together with \thmref{thm:smooth_lb}, implies that both strong convexity and smoothness are necessary for the number of communication rounds to scale logarithmically with the required accuracy $\epsilon$. We emphasize that this is true even if we allow the machines unbounded computational power, to perform arbitrarily many operations satisfying Assumption \ref{assump:dyn}. Moreover, a preliminary analysis indicates that performing accelerated gradient descent on smoothed versions of the local functions (using Moreau proximal smoothing, e.g. \cite{nesterov2005smooth,yu2013better}), can match these lower bounds up to log factors\footnote{Roughly speaking, for any $\gamma>0$, this smoothing creates a $\frac{1}{\gamma}$-smooth function which is $\gamma$-close to the original function. Plugging these into the guarantees of accelerated gradient descent and tuning $\gamma$ yields our lower bounds. Note that, in order to execute this algorithm each machine must be sufficiently powerful to obtain the gradient of the Moreau envelope of its local function, which is indeed the case in our framework.}. We leave a full formal derivation (which has some subtleties) to future work.


The full proof of \thmref{thm:nonsmooth_lb} appears in the supplementary
material. The proof idea relies on the following construction: Assume that we fix the number of communication rounds to be $T$, and (for simplicity) that $T$ is even and the number of machines is $2$. Then we use local functions of the form
\begin{align*}
  F_1(\bw) &= \frac{1}{\sqrt{2}}\absval{b-w_1}+\frac{1}{\sqrt{2(T+2)}}\circpar{\absval{w_2-w_3}+\absval{w_4-w_5} +\dots+ \absval{w_{T}-w_{T+1}}}+\frac{\lambda}{2}\norm{\bw}^2\\
  F_{2}(\bw) &= \frac{1}{\sqrt{2(T+2)}}\circpar{\absval{w_1-w_2}+\absval{w_3-w_4} +\dots+ \absval{w_{T+1}-w_{T+2}}}+\frac{\lambda}{2}\norm{\bw}^2,
\end{align*}
where $b$ is a suitably chosen parameter. It is easy to verify that both local functions are $\lambda$-strongly convex and $(1+\lambda)$-Lipschitz continuous over the unit Euclidean ball. Similar to the smooth case, we argue that after $T$ communication rounds, the resulting points $\bw$ computed by machine $1$ will be non-zero only on the first $T+1$ coordinates, and the points $\bw$ computed by machine $2$ will be non-zero only on the first $T$ coordinates. As in the smooth case, these functions allow us to 'control' the progress of any algorithm which satisfies Assumption \ref{assump:dyn}. 

Finally, although the result is in the unrelated setting, it is
straightforward to have a similar construction in a `$\delta$-related'
setting, by multiplying $F_1$ and $F_2$ by $\delta$. The resulting two
functions have their gradients and subgradients at most $\delta$-different
from each other, and the construction above leads to a lower bound of
$\Omega(\delta/\epsilon)$ for convex Lipschitz functions, and
$\Omega(\delta\sqrt{1/\lambda \epsilon})$ for $\lambda$-strongly convex
Lipschitz functions. In terms of upper bounds, we are actually unaware of any relevant
algorithm in the literature adapted to such a setting, and the question of attainable performance here remains wide open.

\section{One Round of Communication}

In this section, we study what lower bounds are attainable without any kind
of structural assumption (such as Assumption \ref{assump:dyn}). This is a
more challenging setting, and the result we present will be limited to
algorithms using a single round of communication round. We note that this
still captures a realistic non-interactive distributed computing scenario,
where we want each machine to broadcast a single message, and a designated
machine is then required to produce an output. In the context of distributed
optimization, a natural example is a one-shot averaging algorithm, where each
machine optimizes its own local data, and the resulting points are averaged
(e.g. \cite{ZWSL11,zhang2013communication}).

Intuitively, with only a single round of communication, getting an
arbitrarily small error $\epsilon$ may be infeasible. The following theorem
establishes a lower bound on the attainable error, depending on the
strong convexity parameter $\lambda$ and the similarity measure $\delta$
between the local functions, and compares this with a `trivial' zero-communication algorithm, which just returns the optimum of a single local function:
\begin{theorem}\label{thm:comm_lb}
  For any even number $m$ of machines, any dimension $d$ larger than some numerical constant, any $\delta\geq 3\lambda>0$, and any (possibly randomized)
  algorithm which communicates at most $d^2/128$ bits in a single round of communication,
  there exist $m$ quadratic functions over $\reals^d$, which are
  $\delta$-related, $\lambda$-strongly convex and $9\lambda$-smooth, for which the following hold for some positive numerical constants
  $c,c'$:
  \begin{itemize}[leftmargin=*]
    \item The point $\hat{\bw}$ returned by the algorithm satisfies
    \[
    \E\left[F(\hat{\bw})-\min_{\bw\in\reals^d}F(\bw)\right]
~\geq~
  c\frac{\delta^2}{\lambda}
  \] in expectation over the algorithm's randomness.
    \item For any machine $j$, if $\hat{\bw}_j=\arg\min_{\bw\in\reals^d}F_j(\bw)$, then $F(\hat{\bw}_j)-\min_{\bw\in\reals^d}F(\bw) \leq
        c'\delta^2/\lambda$.
  \end{itemize}
\end{theorem}
The theorem shows that unless the communication budget is extremely large
(quadratic in the dimension), there are functions which cannot be optimized
to non-trivial accuracy in one round of communication, in the sense that the
same accuracy (up to a universal constant) can be obtained with a `trivial'
solution where we just return the optimum of a single
local function. This complements an earlier result in
\cite{shamir2014communication}, which showed that a \emph{particular} one-round
algorithm is no better than returning the optimum of a local function, under the stronger assumption that the local functions are not merely $\delta$-related, but are actually the average loss over some randomly partitioned data.

The full proof appears in the supplementary material, but we sketch the main
ideas below. As before, focusing on the case of two machines, and assuming
machine $2$ is responsible for providing the output, we use
\[
  F_1(\bw) = 3\lambda\bw^\top\left(\left(I+\frac{1}{2c\sqrt{d}}M\right)^{-1}-\frac{1}{2}I\right)\bw
\]
\[
 F_2(\bw) = \frac{3\lambda}{2}\norm{\bw}^2-\delta \be_j,
\]
where $M$ is essentially a randomly chosen $\{-1,+1\}$-valued $d\times d$ symmetric matrix with spectral norm at most $c\sqrt{d}$, and $c$ is a suitable constant. These functions can be shown to be  $\delta$-related as well as $\lambda$-strongly
convex. Moreover, the optimum of $F(\bw)=\frac{1}{2}(F_1(\bw)+F_2(\bw))$
equals
\[
\bw^* = \frac{\delta}{6\lambda}\left(I+\frac{1}{2c\sqrt{d}}M\right)\be_j.
\]
Thus, we see that the optimal point $\bw^*$ depends on the $j$-th column of
$M$. Intuitively, the machines need to approximate this column, and this is
the source of hardness in this setting: Machine $1$ knows $M$ but not $j$,
yet needs to communicate to machine $2$ enough information to construct its
$j$-th column. However, given a communication budget much smaller than the
size of $M$ (which is $d^2$), it is difficult to convey enough
information on the $j$-th column without knowing what $j$ is. Carefully
formalizing this intuition, and using some information-theoretic tools,
allows us to prove the first part of \thmref{thm:comm_lb}. Proving the second
part of \thmref{thm:comm_lb} is straightforward, using a few
computations.

\section{Summary and Open Questions}

In this paper, we studied lower bounds on the number of communication rounds
needed to solve distributed convex learning and optimization problems, under
several different settings. Our results indicate that when the local
functions are unrelated, then regardless of the local machines' computational
power, many communication rounds may be necessary (scaling polynomially with
$1/\epsilon$ or $1/\lambda$), and that the worst-case optimal algorithm (at
least for smooth functions) is just a straightforward distributed
implementation of accelerated gradient descent. When the functions are
related, we show that the optimal performance is achieved by the algorithm of
\cite{zhang2015communication} for quadratic and strongly convex 
functions, but designing optimal algorithms for more general functions
remains open. Beside these results, which required a certain mild structural
assumption on the algorithm employed, we also provided an assumption-free
lower bound for one-round algorithms, which implies that even for strongly
convex quadratic functions, such algorithms can sometimes only provide
trivial performance.

Besides the question of designing optimal algorithms for the remaining
settings, several additional questions remain open. First, it would be
interesting to get assumption-free lower bounds for algorithms with multiple
rounds of communication. Second, our work focused on \emph{communication}
complexity, but in practice the \emph{computational} complexity of the local
computations is no less important. Thus, it would be interesting to
understand what is the attainable performance with simple, runtime-efficient
algorithms. Finally, it would be interesting to study lower bounds for other
distributed learning and optimization scenarios.

\paragraph{Acknowledgments:}
This research is supported in part by an FP7 Marie Curie CIG grant, the Intel ICRI-CI Institute, and Israel
Science Foundation grant 425/13. We thank Nati Srebro for several helpful discussions and insights.

\bibliographystyle{plain}
\bibliography{mybib}

\newpage

\appendix
\section{Proofs}

\subsection{Proof of \thmref{thm:smooth_lb}} \label{proof:thm:smooth_lb}

The proof of the theorem is based on splitting the machines into two sub-groups
of the same size, each of which is assigned with a finite dimensional
restriction of $F_1$ and $F_2$ (see \eqref{eq:fdef}), and tracing the maximal number of non-zero
coordinates for vectors in $W_j$, the set of feasible points.\\

Recall that $F_i$ are defined as follows:
\begin{gather}
	F_1(\bw)= \frac{\delta(1-\lambda)}{4}\bw^\top A_1 \bw-\frac{\delta(1-\lambda)}{2}\be_1^\top \bw + \frac{\lambda}{2}\norm{\bw}^2 \notag\\
   F_2(\bw)= \frac{\delta(1-\lambda)}{4}\bw^\top A_2\bw + \frac{\lambda}{2}\norm{\bw}^2,~~~\text{where}\notag\\
		A_1 = \left[~
  \begin{matrix}
  1 & 0 & 0 & 0 & 0 & 0 &  \hdots \\
  0 & 1 & -1 & 0 & 0 & 0 &  \hdots \\
  0 & -1 & 1 & 0 & 0 & 0 &  \hdots \\
  0 & 0 & 0 & 1 & -1 & 0 &  \hdots \\
  0 & 0 & 0 & -1  & 1 & 0 &  \hdots \\
  \vdots & \vdots & \vdots & \vdots & \vdots & \vdots & \vdots
  \end{matrix}
  ~\right],~~~
		A_2 = \left[~
  \begin{matrix}
  1 & -1 & 0 & 0 & 0 & 0 & \hdots \\
  -1 & 1 & 0 & 0 & 0 & 0 & \hdots \\
  0 & 0 & 1 & -1 & 0 & 0 & \hdots \\
  0 & 0 & -1 & 1 & 0 & 0 & \hdots \\
  0 & 0 & 0 & 0 & 1 & -1 & \hdots \\
  0 & 0 & 0 & 0 & -1 & 1 & \hdots \\
  \vdots & \vdots & \vdots & \vdots & \vdots & \vdots &  \vdots
  \end{matrix}
  ~\right]\notag
\end{gather}
Formally speaking, we consider the matrices $A_1,A_2$ as infinite in size, so that each $F_i$ is defined over $\ell^2(\reals)$, the space of square-summable
sequences. To derive lower bounds in $\reals^d$, we
consider the following restrictions of $F_i$ and $F$:
\begin{align*}
	[F_i]_d(\bw) &:= F_i(w_1,w_2,\dots,w_d,0,0,\dots),\quad\bw\in \reals^d\\
	[F]_d(\bw) &:= \frac{[F_1]_d(\bw)+[F_2]_d(\bw)}{2}
\end{align*}
Note that $[F_i]_d(\bw)$ and $[F]_d(\bw)$ produce the same values as $F_i(\bw)$ and $F(\bw)$ do for
vectors such that $\bw_i=0$ for all $i\ge d$. Similarly, we define the $d\times d$ leading principal submatrix of $A_i$ by $[A_{i}]_{d}$.\\

We assign half of the machines with $[F_1]_d$, and the other half with $[F_2]_d$.
To prove the theorem, we need the following lemma, which formalizes the
intuition described in the main paper. Let
\[
E_{0,d}=\{\mathbf{0}\}~~,~~
E_{T,d}=\algspan\{\be_{1,d},\ldots,\be_{T,d}\},
\]
where $\be_{i,d}\in\reals^d$ denote the standard unit vectors. Then, the following
holds:
\begin{lemma}\label{lemma:dyn_smooth}
Suppose all the sets of feasible points satisfy $W_j\subseteq E_{T,d}$ for some $T\le d-1$, then under assumption \ref{assump:dyn}, right after the next communication round we have $W_{j}\subseteq E_{T+1,d}$.
\end{lemma}
\begin{proof}
Recall that by Assumption \ref{assump:dyn}, each machine can compute new points $\bw$ that
satisfy the following for some $\gamma,\nu\geq 0$ such that $\gamma+\nu> 0$:
\begin{align*}
\gamma\bw &+ \nu \nabla [F_i]_d(\bw) \in  \algspan\Big\{\bw'~,~\nabla [F_i]_d(\bw')~,~(\nabla^2 [F_i]_d(\bw')+D)\bw''~,~(\nabla^2 [F_i]_d(\bw')+D)^{-1}\bw'' ~\Big|~\\
&
	\bw',\bw'' \in W_j~~,~~ D~\text{diagonal}~~,~~\nabla^2 [F_i]_d(\bw')~\text{exists}~~,~~(\nabla^2 [F_i]_d(\bw')+D)^{-1}~\text{exists}\Big\}.
\end{align*}
We now analyze the state of the sets of feasible points prior to the next communication round. Assume that $T$ is an odd number, i.e., assume  $T=2k+1$ for some $k=0,1,\dots$. The proof for the case where $T$ is even follows similar lines.
Note that for any $\bw',\bw''\in W_j$, we have
\begin{align*}
	&\nabla [F_1]_d(\bw') = \frac{\delta(1-\lambda)}{2} [A_{1}]_{d} \bw' -\frac{\delta(1-\lambda)}{2}\be_1 + \frac{\lambda}{2}\bw' \subseteq E_{2k+1,d} \\
	&(\nabla^2 [F_1]_d(\bw')+D)\bw'' = \circpar{\frac{\delta(1-\lambda)}{2} [A_{1}]_{d} + D +\lambda I }\bw'' \subseteq E_{2k+1,d}\\
	&(\nabla^2 [F_1]_d(\bw')+D)^{-1}\bw'' = \circpar{\frac{\delta(1-\lambda)}{2} [A_{1}]_{d} + D +\lambda I}^{-1}\bw'' \subseteq E_{2k+1,d}
\end{align*}
For any viable diagonal matrix D. Therefore, since $W_j\subseteq	E_{2k+1,d}$, we have that the first point generated by machines which hold $[F_1]_d(\bw)$ must satisfy
\begin{align*}
\gamma\bw &+ \nu \nabla [F_1]_d(\bw) \in E_{2k+1,d}
\end{align*}
for $\gamma,\nu$ as stated in the assumption. That is,
\begin{align*}
	\left(\frac{\delta(1-\lambda)}{2} \nu [A_{1}]_{d}+ (\gamma +\frac{\nu\lambda}{2}) I\right) \bw - \frac{\delta(1-\lambda)}{2}\be_1 \in E_{2k+1,d}
\end{align*}
Which implies,
\begin{align*}
	\underbrace{\left(\frac{\delta(1-\lambda)}{2} \nu [A_{1}]_{d}+ (\gamma +\frac{\nu\lambda}{2}) I\right)}_{H} \bw \in E_{2k+1,d}
\end{align*}
Since $[A_{1}]_{d}$ is positive semidefinite, it holds that $H$ is invertible. Also, $[A_{1}]_{d}, H$ and $H^{-1}$ admit the same partitions into  $1\times1$ and $2\times2$ blocks on the diagonal, thus $H^{-1}E_{2k+1,d}\subseteq E_{2k+1,d}$, yielding $\bw \in E_{2k+1,d}$. Inductively extending the latter argument shows that, in the absence of any communication rounds, all the machines whose local function is $[F_1]_d(\bw)$ are `stuck' in $E_{2k+1,d}$.\\
\\
As for machines which contain $[F_2]_d(\bw)$, we have that for all $\bw',\bw''\in W_j$
\begin{align*}
	&\nabla [F_2]_d(\bw') = \frac{\delta(1-\lambda)}{2} [A_{2}]_{d} \bw' + \frac{\lambda}{2}\bw' \subseteq E_{2k+2,d} \\
	&(\nabla^2 [F_2]_d(\bw')+D)\bw'' = \circpar{\frac{\delta(1-\lambda)}{2} [A_{2}]_{d} + D +\lambda I }\bw'' \subseteq E_{2k+2,d}\\
	&(\nabla^2 [F_2]_d(\bw')+D)^{-1}\bw'' = \circpar{\frac{\delta(1-\lambda)}{2} [A_{2}]_{d} + D +\lambda I}^{-1}\bw'' \subseteq E_{2k+2,d}
\end{align*}
For any viable diagonal matrix D. Therefore, the first generated point by these machines must satisfy,
\begin{align*}
\gamma\bw &+ \nu \nabla [F_1]_d(\bw) \in E_{2k+2,d}
\end{align*}
for appropriate $\gamma,\nu$. Hence,
\begin{align*}
	\left(\frac{\delta(1-\lambda)}{2} \nu [A_{2}]_{d}+ (\gamma +\frac{\nu\lambda}{2}) I\right) \bw \in E_{2k+2,d}
\end{align*}
Similarly to the previous case this implies that $\bw \in E_{2k+2,d}$.  It is now left to show that these machines cannot make further progress beyond $E_{2k+2,d}$ without communicating. To see this, note that for all $\bw',\bw''\in E_{2k+2,d}$ we have,
\begin{align*}
	&\nabla [F_2]_d(\bw') = \frac{\delta(1-\lambda)}{2} [A_{2}]_{d} \bw' + \frac{\lambda}{2}\bw' \subseteq E_{2k+2,d} \\
	&(\nabla^2 [F_2]_d(\bw')+D)\bw'' = \circpar{\frac{\delta(1-\lambda)}{2} [A_{2}]_{d} + D +\lambda I }\bw'' \subseteq E_{2k+2,d}\\
	&(\nabla^2 [F_2]_d(\bw')+D)^{-1}\bw'' = \circpar{\frac{\delta(1-\lambda)}{2} [A_{2}]_{d} + DI +\lambda I}^{-1}\bw'' \subseteq E_{2k+2,d}
\end{align*}
This means that all the points which are generated subsequently also lie in $E_{2k+2,d}$, i.e., without communicating , machines whose local function is $[F_2]_d(\bw)$ are stuck in $E_{2k+2,d}$. Finally, executing a communication round updates all the sets of feasible points to be $W_j := E_{2k+2,d}$.
\end{proof}

The following is a direct consequence of a recursive application of \lemref{lemma:dyn_smooth}.
\begin{corollary} \label{cor:dyns}
Under assumption \ref{assump:dyn}, after $T \le d-1$ communication rounds we have
$$W_j\subseteq E_{T+1},\quad j\in\{1,\ldots,m\}$$
\end{corollary}
With this corollary in hand, we now turn to prove the main result. First, we compute the minimizer of the average function $F(\bw)=\frac{\frac{m}{2}F_1(\bw)+\frac{m}{2}F_2(\bw)}{m}$ in $\ell^2(\reals)$, denoted by $\bw^*$, whose
form for even number of machines is simply:
\begin{align*}\label{eq:favdef}
	F(\bw)&= \frac{\delta(1-\lambda)}{8}\bw^\top \left(A_1 + A_{2}\right) \bw - \frac{\delta(1-\lambda)}{4}\be_1^\top \bw   + \frac{\lambda}{2}\norm{\bw}^2
\end{align*}
By first-order optimality condition for smooth convex functions, we have
\begin{align*}
	\left(\frac{\delta(1-\lambda)}{4}\left(A_1 + A_2\right) +\lambda I \right) \bw^* - \frac{\delta(1-\lambda)}{4}\be_1   =0,
\end{align*}
or equivalently,
\[
	\left( A_1 + A_2 + \frac{4\lambda}{\delta(1-\lambda)} I \right) \bw^*   =\be_1
\]
whose coordinate form is as follows
\begin{align}
	\left(2 + \frac{4\lambda}{\delta(1-\lambda)}\right) \bw^*[1] - \bw^*[2] &= 1\label{eq:eqcond1}\\
  \forall~k~,~~\bw^*[k+1] - \left(2 + \frac{4\lambda}{\delta(1-\lambda)}\right) \bw^*[k] + \bw^*[k-1] &= 0.\label{eq:eqcond2}
\end{align}
The optimal solution can be now realized as a geometric sequence $(\zeta^k)_{k=1}^\infty$ for some $\zeta$ as follows:\\
By \eqref{eq:eqcond2}, we must have
\begin{align*}
 \zeta^2- \left(2+ \frac{4\lambda}{\delta(1-\lambda)}\right) \zeta + 1 =0,
\end{align*}
with the smallest root being
\begin{align} \label{equation:smallest_root}
	\zeta&=  \frac{2+ \frac{4\lambda}{\delta(1-\lambda)} - \sqrt{\left(2+ \frac{4\lambda}{\delta(1-\lambda)}\right)^2-4}}{2}=1+ \frac{2\lambda}{\delta(1-\lambda)} - \sqrt{\left(1+ \frac{2\lambda}{\delta(1-\lambda)}\right)^2-1}\nonumber\\
\end{align}
Therefore, this choice of $\zeta$ satisfies \eqref{eq:eqcond2}, and it is
straightforward to verify that it also satisfies \eqref{eq:eqcond1}, hence
$\bw^*$ indeed equals $(\zeta^k)_{k=1}^\infty$. It will be convenient to denote
a continuous range of coordinates of $(\zeta^k)_{k=1}^\infty$ by $\bzeta_{a:b}$ where
$a\in \bN$ and $b\in\bN\cup\infty$. Also, using the following
inequality which holds for $x>1$
\begin{align*}
 x - \sqrt{x^2-1}
&\ge \exp\left(\frac{-2}{\sqrt{\frac{x +1}{x-1}}-1}\right),\quad x>1
\end{align*}
together with \eqref{equation:smallest_root} yields
\begin{align} \label{eq:zeta}
	\zeta\ge  \exp\left(  \frac{-2}{\sqrt{\delta(1/\lambda-1)+1  }-1} \right)  	
\end{align}
%
%
%

We now use this computation (with respect to $F_1,F_2$) to find the minimizer of $[F]_d$, defined as the average function of the finite-dimensional restrictions  $[F_1]_d,[F_2]_d$ actually handed to the machines.
Fix $d\in \bN$ and denote the corresponding minimizer by
\begin{align*}
	\bw_d^*&=\arg\min_{\bw\in\reals^d} [F]_d(\bw)
\end{align*}
Let $\bw_T$ be some point which was obtained after $T\le d-2$ communication
rounds. To bound the sub-optimality of $\bw_T$ from below, observe that
\begin{align*}
	[F]_d(\bw_T) - [F]_d(\bw^*_d)&\ge [F]_d(\bw_T) - [F]_d(\bzeta_{1:{d-1}}) \\
	&= [F]_d(\bw_T) - F(\bw^*) + F(\bw^*)  - [F]_d(\bzeta_{1:{d-1}})\\
	&= \underbrace{F(\bw_T) - F(\bw^*)}_A + \underbrace{F(\bw^*)  - F(\bzeta_{1:{d-1}})}_B
\end{align*}
where the last equality follows from Corollary (\ref{cor:dyns}), according to
which all the coordinates of $\bw_T$, except for the first $T+1\le d-1$, must vanish.
To bound the $A$ term, note that
\begin{align*}
	\norm{\bw_T -\bw^*}^2&\ge
	\sum_{t=T+2}^\infty \zeta^{2t} =  \zeta^{2 (T+1)} \sum_{t=1}^\infty \zeta^{2t}
	=
		\zeta^{2 (T+1)} \norm{\bw^*}^2
\end{align*}
The fact that $F(\bw)$ is $\lambda$-strongly convex implies
\begin{align*}
	F(\bw_T) -F(\bw^*)&\ge \frac{\lambda}{2}\norm{\bw_T-\bw^*}^2 \geq \frac{\lambda\zeta^{2 (T+1)}}{2} \norm{\bw^*}^2.
\end{align*}
Inequality (\ref{eq:zeta}) yields	
\begin{align*}
	F(\bw_T) -F(\bw^*) &\ge
\frac{\lambda}{2}
	  \exp\left(  \frac{-4T-4}{\sqrt{\delta(1/\lambda-1)+1  }-1}  \right)
 \norm{\bw^*}^2
\end{align*}
To bound the $B$ term from below, note that since $F$ is 1-smooth we have
\begin{align*}
	F(\bw^*)  - F(\bzeta_{1:{d-1}})&\ge -\frac{1}{2}\norm{\bw^*-\bzeta_{1:{d-1}}}^2 =
	-\frac{\zeta^{2(d-1)}}{2}\sum_{t=1}^\infty \zeta^{2t} =
	-\frac{\zeta^{2(d-1)}}{2} \norm{\bw^*}^2
\end{align*}
Combining both lower bounds for the terms $A$ and $B$, we get for any $T\le d-2$
\begin{align*}
	[F]_d(\bw_T) - [F]_d(\bw^*_d)&\ge
		\left(\frac{\lambda}{2}
	  \exp\left(  \frac{-4T-4}{\sqrt{\delta(1/\lambda-1)+1  }-1}  \right)
 -\frac{\zeta^{2(d-1)}}{2} \right)\norm{\bw^*}^2
\end{align*}
Picking $d$ sufficiently large, and considering how large the number of communication rounds $T$ must be to make this lower bound less than $\epsilon$, we get
\begin{align*}
T~\geq~\frac{\sqrt{\delta(1/\lambda-1)+1  }-1}{4}  \ln\left(\frac{\lambda\norm{\bw^*}^2}{4\epsilon} \right) -1.
\end{align*}
It is worth mentioning that by computing the exact minimizers of $[F_i]_d$ one
may derive a lower bound such that the choice of $d$ does not depend on the parameters
of the problem, except for the number of communication rounds. Nevertheless,
such analysis requires a more involved reasoning which we find unnecessary for
stating our results.

For the non-strongly convex case, where $\lambda=0$, using Corollary \ref{cor:dyns} and a similar analysis (virtually identical to the proof of Theorem 2.1.7 in \cite{nesterov2004introductory}),
we have that if $T\le\frac{1}{2}(d-1)$, then
\begin{align*}
    F(\bw_T) - F(\bw^*) \ge \frac{3\delta\|w^*\|^2}{32(T+2)^2}
\end{align*}
Therefore, to obtain an $\epsilon$-suboptimal solution for this case, we must
have at least
\begin{align*}
	 \sqrt{\frac{3\delta}{32 \epsilon } }\norm{ \bw^*}-2
\end{align*}
communication rounds, for sufficiently small $\epsilon$.

%

\subsection{Proof of \thmref{thm:nonsmooth_lb}} \label{prf:thm:nonsmooth_lb}

We construct two types of local functions, and provide one of them to
$m/2$ of the machines, and the other function to the other $m/2$ machines, in some arbitrary order. In this case, the average function is simply the average of the two types of local functions.

We will first prove the theorem statement in the strongly convex case, where $\lambda>0$ is given, and then explain how to extract from it the result in the non strongly convex case. 

Fix a natural number $k$ and some $b\in [0,1/\sqrt{k}]$, to be specified later. We define the following local function over the unit ball:
\begin{align} 
	F_{1,k}(\bw) &= \frac{1}{\sqrt{2}}\absval{b-w[1]}+\frac{1}{\sqrt{2k}}\circpar{\absval{w[2]-w[3]}+\absval{w[4]-w[5]} +\dots+ \absval{w[k-2]-w[k-1]}} + \frac{\lambda}{2}\norm{\bw}^2\notag\\
	F_{2,k}(\bw) &= \frac{1}{\sqrt{2k}}\circpar{\absval{w[1]-w[2]}+\absval{w[3]-w[4]} +\dots+ \absval{w[k-1]-w[k]}}+\frac{\lambda}{2}\norm{\bw}^2 \label{def:non_smooth_locals}
\end{align}
For even $k\le d$, and
\begin{align*}
	F_{1,k}(\bw) &= \frac{1}{\sqrt{2}}\absval{b-w[1]}+\frac{1}{\sqrt{2k}}\circpar{\absval{w[2]-w[3]}+\absval{w[4]-w[5]} +\dots+ \absval{w[k-1]-w[k]}}+ \frac{\lambda}{2}\norm{\bw}^2\\
	F_{2,k}(\bw) &= \frac{1}{\sqrt{2k}}\circpar{\absval{w[1]-w[2]}+\absval{w[3]-w[4]} +\dots+ \absval{w[k-2]-w[k-1]}}+\frac{\lambda}{2}\norm{\bw}^2 \nonumber
\end{align*}
otherwise.
Being a sum of convex functions, both local functions are convex, and in fact $\lambda$-strongly convex due to the $\frac{\lambda}{2}\norm{\bw}^2$ term. 
Furthermore, both function are $(1+\lambda)$-Lipschitz continuous over the unit Euclidean ball. To see this, let $\partial (\cdot)$ denote the subdifferential operator and note that
\begin{align*}
	\bg \in \partial \absval{b-w[1]} &\implies \bg \in \conv\{\bsigma_0, -\bsigma_0\}\\
	\bg \in \partial \absval{w[l]-w[l+1]} &\implies \bg \in \conv\{\bsigma_l,-\bsigma_l\}
\end{align*}
where
\begin{align*}
	\bsigma_0 &= (1,0,\dots,0)\\
	\bsigma_l &= (0,\dots,0,\underbrace{1}_{l},\underbrace{-1}_{l+1},0,\dots,0).
\end{align*}
Assume for a moment that $\lambda=0$, then by the linearity of the sub-differential operator that
\begin{align*}
	&\forall \bg \in \partial F_{1,k}(\bw),~\norm{\bg}\le  \sqrt{\frac{1}{2} + \frac{k-1}{2k} }\le1\\
	&\forall  \bg \in \partial F_{2,k}(\bw),~\norm{\bg}\le \sqrt{\frac{1}{2} }
\end{align*}
which shows that, for $\lambda=0$, both functions are 1-Lipschitz. For $\lambda>0$,
note that $\frac{\lambda}{2}\norm{\bw}^2$ is $\lambda$-Lipschitz over the unit
ball and $\lambda$-strongly convex. Therefore, using the linearity of
the sub-differential operator again, we see that both $F_i$ are $(1+\lambda)$-Lipschitz
and $\lambda$-strongly convex functions over the unit ball.

Similar to the smooth case, the following lemma shows that, no matter how the subgradients are
chosen, at each iteration at most one non-zero coordinate may be gained.
\begin{lemma}\label{lemma:dyn_nonsmooth}
Suppose all the sets of feasible points satisfy $W_j\subseteq E_{T,d}$ for some $T\le d-1$. Then under assumption \ref{assump:dyn}, right after the next communication round we have $W_{j}\subseteq E_{T+1,d}$.
\end{lemma}
\begin{proof}
Recall that by Assumption \ref{assump:dyn} (modified for the non-differentiable case), each machine can compute new points $\bw$ that
satisfy the following for some $\gamma,\nu\geq 0$ such that $\gamma+\nu> 0$:
\begin{align*}
\gamma\bw &+ \nu \bg_{i,k}(\bw) \in  \algspan\Big\{\bw'~,~\bg_{i,k}(\bw')~,~(\nabla^2 F_{i,k}(\bw')+D)\bw''~,~(\nabla^2 F_{i,k}(\bw')+D)^{-1}\bw'' ~\Big|~\\
&
	\bw',\bw'' \in W_j~,~ \bg_{i,k}(\bw')\in \partial F_{i,k}(\bw')~,~D~\text{diagonal}~,~\nabla^2 F_{i,k}(\bw')~\text{exists}~,~(\nabla^2 F_{i,k}(\bw')+D)^{-1}~\text{exists}\Big\}.
\end{align*}

We now analyze the state of the sets of feasible points prior to the next communication round. Assume that $T$ is an odd number, i.e., assume  $T=2p+1$ for some $p\in \bN\cup\{0\}$. We show that as long as no communication round has been executed, it must hold that $W_j\subseteq E_{T,d}$ for machines whose local function is $F_1$, and that $W_j\subseteq E_{T+1,d}$ for machines whose local function is $F_2$. The case where $T$ is even follows a similar line.

Let
\[
E_{0,d}=\{\mathbf{0}\}~~,~~
E_{T,d}=\algspan\{\be_{1,d},\ldots,\be_{T,d}\}
\]
where $\be_{i,d}\in\reals^d$ denote the standard unit vectors.
First, we prove the claim for machines whose local function is $F_{1,k}$.
In which case, for any $\bw',\bw''\in E_{2p+1,d}$, it holds that
\begin{align*}
	&\bg_{1,k}(\bw') \subseteq \conv\{\pm\bsigma_{2l}~|~l=0,\dots,p\} \subseteq E_{2p+1,d} \\
	&(\nabla^2 F_{1,k}(\bw')+D)\bw'' = \circpar{\lambda I + D}\bw'' \subseteq E_{2p+1,d}\\
	&(\nabla^2 F_{1,k}(\bw')+D)^{-1}\bw'' = \circpar{\lambda I + D}^{-1}\bw'' \subseteq E_{2p+1,d}
\end{align*}
For any viable diagonal matrix D. Therefore,  we have that the first point generated by machines which hold $F_{1,k}$ must satisfy
\begin{align} \label{eq:non_smooth_non_strongly_F_1}
\gamma\bw &+ \nu \bg_{1,k}(\bw) \in E_{2p+1,d}
\end{align}
for $\gamma,\nu$ as stated in Assumption (\ref{assump:dyn}). 
Note that, if $\nu=0$ (which by assumption means that $\gamma>0$) then clearly $\bw\in E_{2p+1,d}$. As for $\nu\neq0$, suppose by
contradiction that $\bw\notin E_{2p+1,d}$. That is, assume that there exists some $j>2p+1$ 
such that $w[j]\neq0$. First, if the absolute value terms in $F_{1,k}$
do not involve $w[j]$, e.g., when $j=d$ and $d$ is even, we have $\bg_{1,k}(\bw)[j] = \lambda w[j]$. In this case, 
by \eqref{eq:non_smooth_non_strongly_F_1} we have 
$$\gamma w[j]+\nu\lambda w[j]=(\gamma +\nu\lambda )w[j]= 0,$$
and since $\nu\lambda>0$, this implies that $w[j]=0$ -- a contradiction! Thus, 
it remains to consider the cases of either odd $d$ or $j\neq d$. In both of these cases, $w[j]$ appears in one of the absolute value terms in
$F_{1,k}$, either as $|w[j-1]-w[j]|$ or $|w[j]-w[j+1]|$ (depending on whether $j$ is odd or even). 

Let $l>p$ be such that either $2l=j$ or $2l+1=j$, depending on the parity of $j$.
We note that any valid subgradient must satisfy
\begin{align*}
\bg_{1,k}(\bw)[2l] &= \frac{2\alpha-1}{\sqrt{2k}}+ \lambda w[2l] \\
\bg_{1,k}(\bw)[2l+1] &= \frac{1-2\alpha}{\sqrt{2k}} + \lambda w[2l+1]
\end{align*}
for some $\alpha\in[0,1]$, such that if $w[2l]-w[2l+1]\neq0$ then
\begin{align} \label{eq:sign_F_1}
	\operatorname{sgn}\circpar{w[2l]-w[2l+1]} = \operatorname{sgn}\circpar{\frac{2\alpha-1}{\sqrt{2k}}},
\end{align}
where $\operatorname{sgn}()$ is the sign function.
Rearranging terms in \eqref{eq:non_smooth_non_strongly_F_1} 
and using the facts that coordinates $2l,2l+1$ are always zero in $E_{2p+1,d}$, as well as $\gamma+\nu\lambda\ge\nu\lambda>0$, we get
\begin{align} \label{eq:alpha_1}
	w[2l] &= \frac{-\nu(2\alpha-1)}{\sqrt{2k}(\gamma+\nu\lambda)}\\
	w[2l+1] &= \frac{\nu(2\alpha-1)}{\sqrt{2k}(\gamma+\nu\lambda)}\nonumber
\end{align}
Therefore,
\begin{align} \label{eq:alpha_2}
w[2l]-w[2l+1]=\frac{-2\nu(2\alpha-1)}{\sqrt{2k}(\gamma+\nu\lambda)}
\end{align}
which implies
\begin{align*}
	\operatorname{sgn}\circpar{w[2l]-w[2l+1]} = \operatorname{sgn}\circpar{\frac{-(2\alpha-1)}{\sqrt{2k}}},
\end{align*}
contradicting \eqref{eq:sign_F_1}. Hence, we must have $w[2l]-w[2l+1]=0$, in which case \eqref{eq:alpha_2} implies $\alpha=1/2$. Thus, by \eqref{eq:alpha_1}
\begin{align*}
	w[2l]=w[2l+1]=0,
\end{align*}
which contradicts the assumption that either $w[j]$ (and hence $w[2l]$ or $w[2l+1]$) is not zero.
Thus, we have shown that $\bw\in E_{2p+1,d}$, for the first point generated by machines holding $F_{1,k}$. Repeating the argument,
we get that any point generated by those machines, in the absence of any communication rounds, is `stuck' in $E_{2p+1,d}$. 

We now turn to prove the claim for machines whose local function is $F_{2,k}$, using an almost identical argument, which we provide below for completeness.
For these functions, we assume that initially $W_j\subseteq E_{2p+1,d}$, and will show any additional points computed locally by the machines must be in $E_{2p+2,d}$.
We begin by noting that for any $\bw',\bw''$ in $E_{2p+2,d}$ (and in particular $E_{2p+1,d}$), it holds that
\begin{align*}
	&\bg_{2,k}(\bw') \subseteq \conv\{\pm\bsigma_{2l+1}~|~l=0,\dots,p\} \subseteq E_{2p+2,d} \\
	&(\nabla^2 F_{2,k}(\bw')+D)\bw'' = \circpar{\lambda I + D}\bw'' \subseteq E_{2p+2,d}\\
	&(\nabla^2 F_{2,k}(\bw')+D)^{-1}\bw'' = \circpar{\lambda I + D}^{-1}\bw'' \subseteq E_{2p+2,d}
\end{align*}
For any viable diagonal matrix D. Therefore, we have that the first point generated by machines which hold $F_{2,k}$ must satisfy
\begin{align} \label{eq:non_smooth_non_strongly_F_2}
\gamma\bw &+ \nu \bg_{2,k}(\bw) \in E_{2p+2,d}
\end{align}
for $\gamma,\nu$ as stated in the assumption. 
Note that, if $\nu=0$ then clearly $\bw\in E_{2p+2,d}$. As for $\nu\neq0$, suppose by
contradiction that $\bw\notin E_{2p+2,d}$. That is, assume that there exists some $j>2p+2$ 
such that $w[j]\neq0$. 
First, if the absolute value terms in $F_{2,k}$
do not involve $w[j]$, e.g., when $j=d$ and $d$ is odd, we have $\bg_{2,k}(\bw)[j] = \lambda w[j]$. In this case, 
by \eqref{eq:non_smooth_non_strongly_F_2} we have 
$$\gamma w[j]+\nu\lambda w[j]=(\gamma +\nu\lambda )w[j]= 0,$$
and since $\nu\lambda>0$, this implies that $w[j]=0$ -- a contradiction! Thus, 
it remains to consider the cases of either even $d$ or $j\neq d$. In both of these cases, $w[j]$ appears in one of the absolute value terms in
$F_{2,k}$, either as $|w[j-1]-w[j]|$ or $|w[j]-w[j+1]|$ (depending on whether $j$ is odd or even). 

Let $l>p$ be such that either $2l+1=j$ or $2l+2=j$, depending on the parity of $j$.
We note that any valid subgradient must satisfy
Any valid subgradient must satisfy
\begin{align*}
\bg_{2,k}(\bw)[2l+1] &= \frac{2\alpha-1}{\sqrt{2k}}+ \lambda w[2l+1] \\
\bg_{2,k}(\bw)[2l+2] &= \frac{1-2\alpha}{\sqrt{2k}} + \lambda w[2l+2]
\end{align*}
for some $\alpha\in[0,1]$, such that if $w[2l+1]-w[2l+2]\neq0$ then
\begin{align} \label{eq:sign_F_2}
	\operatorname{sgn}\circpar{w[2l+1]-w[2l+2]}=\operatorname{sgn}\circpar{\frac{2\alpha-1}{\sqrt{2k}}}
\end{align}
Rearranging terms in \eqref{eq:non_smooth_non_strongly_F_2} and using the fact that $\gamma+\nu\lambda\ge\nu\lambda>0$, we get
\begin{align} \label{eq:alpha_1_F_2}
	w[2l+1] &= \frac{-\nu(2\alpha-1)}{\sqrt{2k}(\gamma+\nu\lambda)}\\
	w[2l+2] &=  \frac{\nu(2\alpha-1)}{\sqrt{2k}(\gamma+\nu\lambda)}\nonumber
\end{align}
Therefore,
\begin{align} \label{eq:alpha_2_F_2}
w[2l+1]-w[2l+2]=\frac{-2\nu(2\alpha-1)}{\sqrt{2k}(\gamma+\nu\lambda)}
\end{align}
which implies
\begin{align*}
		\operatorname{sgn}\circpar{w[2l+1]-w[2l+2]}=\operatorname{sgn}\circpar{\frac{-(2\alpha-1)}{\sqrt{2k}}} 
\end{align*}
a contradiction to \eqref{eq:sign_F_2}. Hence, we must have $w[2l+1]-w[2l+2]=0$, in which case \eqref{eq:alpha_2_F_2} implies $\alpha=1/2$. Thus, by \eqref{eq:alpha_1_F_2}
\begin{align*}
	w[2l+1]=w[2l+2]=0
\end{align*}
which contradicts our assumption that $w[j]$ (and hence either $w[2l+1]$ or $w[2l+2]$ is not zero).
Thus, we have shown that $\bw\in E_{2p+2,d}$. As before, repeating the argument together with the assumption that $W_j\subseteq E_{2p+2,d}$ shows that, in the absence of any communication rounds, all the machines whose local function is $F_{2,k}$ are 'stuck' in $E_{2p+2,d}$. Therefore, before the next communication round, $W_j\subseteq E_{2p+2,d}$ for all machines $j$ holding $F_{2,j}$. Moreover, as shown earlier, $W_j\subseteq E_{2p+1,d}$ for all machines holding $F_{1,k}$. Therefore, after the next communication round, $W_j\subseteq E_{2p+2,d}$ for any machine $j$.
\end{proof}

Repeatedly applying \lemref{lemma:dyn_nonsmooth}, we get the following corollary:
\begin{corollary} \label{cor:dyns_2}
Under assumption \ref{assump:dyn}, after $T \le d-1$ communication rounds we have
$$W_j\subseteq E_{T+1},\quad j\in \{1,\ldots,m\}$$
\end{corollary}
With this corollary in hand, we now turn to establish the main result, namely, bounding
from below the optimality of points in $W_j$ after $T$ communication
rounds. Choosing the dimension $d$ such that $T\le d-2$, we employ the local functions defined in
\eqref{def:non_smooth_locals} with $k=T+2$. In which case, the average function is
\begin{align*}
	F(\bw)=&\frac{1}{2}F_{1,T+2}(\bw)+ \frac{1}{2}F_{2,T+2}(\bw) = \frac{1}{2\sqrt{2}}\absval{b-w[1]}+\frac{1}{2\sqrt{2(T+2)}}
	\sum_{i=1}^{T+1} \absval{w[i]-w[i+1]} + \frac{\lambda}{2}\norm{\bw}^2
\end{align*}

The key ingredient in deriving the lower bound is Corollary (\ref{cor:dyns_2}),
according to which after $T$ communication rounds, all but the first $T+1$ coordinates must be zero, in
particular $w[T+2]=0$. Using this and the triangle inequality, we have
\begin{align*}
	F(\bw) &\ge \frac{1}{2\sqrt{2}}\absval{b-w[1]}+\frac{1}{2\sqrt{2(T+2)}}
	\absval{w[1]-w[T+2]} + \frac{\lambda}{2}\norm{\bw}^2\\
	&= \frac{1}{2\sqrt{2}}\absval{b-w[1]}+\frac{1}{2\sqrt{2(T+2)}}
	\absval{w[1]} + \frac{\lambda w[1]^2}{2}
\end{align*}
for all $\bw$ in $W_j$. Therefore, we can lower bound the objective value of the algorithm's output by
\[
\min_{w\in\reals} \circpar{\frac{1}{2\sqrt{2}}\absval{b-w}+\frac{1}{2\sqrt{2(T+2)}}
	\absval{w} + \frac{\lambda w^2}{2}}
\]
On the flip side, the minimal value of $F(\bw)$ over the unit Euclidean ball can 
be upper bounded by $F(\bw_b)$ for some $b\le \frac{1}{\sqrt{T+2}}$, where
\begin{align*}
	\bw_b = (\underbrace{b,\dots,b}_{T+2 \text{ times}},0,\dots,0)
\end{align*}
Putting both bounds together yields,
\begin{align}
	\min_{\bw \in W_j}F(\bw) - \min_{\norm{\bw}\le1} F(\bw) &\ge
	\min_{\bw \in W_j}F(\bw) - F(\bw_b)\notag\\ &\ge
	\min_{w\in\reals} \circpar{\frac{1}{2\sqrt{2}}\absval{b-w}+\frac{1}{2\sqrt{2(T+2)}}
	\absval{w} + \frac{\lambda w^2}{2}} - \frac{\lambda(T+2)b^2}{2}
\label{eq:lbnn}
\end{align}


Assuming $T\geq \frac{1}{2\lambda}-2$ (so that $\lambda\geq \frac{1}{2(T+2)}$), we take
\begin{align*}
	b& =\frac{1}{2\lambda(T+2)\sqrt{2(T+2)}}	
\end{align*}
(note again that $\bw_b$ is indeed in the unit ball for this regime of $\lambda$ and $T$). In this case, the minimal $w$ in \eqref{eq:lbnn} is $\frac{1}{2\lambda(T+2)\sqrt{2(T+2)}}$, so we get a suboptimality lower bound of
\begin{align} \label{def:nonsmooth_final_lb}
&\circpar{0+\frac{1}{2\sqrt{2(T+2)}}
	\absval{\frac{1}{2\lambda(T+2)\sqrt{2(T+2)}}} + \frac{1}{8\lambda((T+2)\sqrt{2(T+2)}^2)}} - \frac{1}{16\lambda(T+2)^2}\nonumber\\
&~~~~~~\geq \circpar{\frac{1}{8\lambda(T+2)^2}+0}- \frac{1}{16\lambda(T+2)^2}\nonumber\\
&~~~~~~=	 \frac{1}{16\lambda(T+2)^2}
\end{align}
This bound holds in particular for any $T\ge \left\lceil\frac{1}{2\lambda}-2\right\rceil$. If the number of communication rounds $T$ is less than $\left\lceil\frac{1}{2\lambda}-2\right\rceil$, then clearly we cannot do better than with $\left\lceil\frac{1}{2\lambda}-2\right\rceil$ communication rounds. Therefore, for any number of communication rounds $T$, the suboptimality is at least
\begin{align*} 
\min\left\{\frac{1}{16\lambda\left(\left\lceil\frac{1}{2\lambda}-2\right\rceil+2\right)^2}~,~\frac{1}{16\lambda(T+2)^2}\right\}
\end{align*}
Therefore, for any $\epsilon\in \left(0,\frac{1}{16\lambda\left(\left\lceil\frac{1}{2\lambda}-2\right\rceil+2\right)^2}\right]$, we would need at least $T\geq \sqrt{\frac{1}{16\lambda \epsilon}}-2$ communication rounds to get an $\epsilon$-suboptimal solution. This implies the theorem statement for $\lambda$-strongly convex functions.\\

Finally, we treat the case where the local functions are not required to be strongly convex. In this setting, for proving a lower bound, we can use the same construction as in \eqref{def:non_smooth_locals},
where we are free to choose any $\lambda$. In particular, let us choose $\lambda = \frac{1}{2(T+2)}$, and apply the lower bound derived above (note that 
in this case the condition $T\geq \frac{1}{2\lambda}-2$ trivially holds). 
Plugging in it into (\ref{def:nonsmooth_final_lb}), we establish 
that for any number of communication rounds $T$, the suboptimality is at least
\begin{align*} \label{def:nonsmooth_final_lb}
\frac{1}{8(T+2)}.
\end{align*}
Considering how large $T$ must be to make this smaller than some $\epsilon$, we 
get that $T$ must be at least $\frac{1}{8\epsilon}-2$. 

\subsection{Proof of \thmref{thm:comm_lb}}

As usual, we construct two functions $F_1,F_2$, and provide $F_1$ to
$m/2$ of the machines, and $F_2$ to the other $m/2$ machines, in some arbitrary order, such that the machine designated to provide the output receives $F_2$. Note that the average function $F$ is simply
$\frac{1}{2}(F_1(\bw)+F_2(\bw))$.

Let $c$ be a certain positive numerical constant (whose value corresponds to
$c$ in \lemref{lem:tao} below). Given some symmetric $M\in \{-1,+1\}^{d\times d}$,
where $\norm{M}\leq c\sqrt{d}$, and $j\in \{\lceil d/2 \rceil,\ldots,d\}$, define
  \[
  F_1(\bw) = 3\lambda\bw^\top\left(\left(I+\frac{1}{2c\sqrt{d}}M\right)^{-1}-\frac{1}{2}I\right)\bw
  \]
  \[
  F_2(\bw) = \frac{3\lambda}{2}\norm{\bw}^2-\delta \be_j,
  \]

The average $F$ of $F_1,F_2$ equals
  \[
  F(\bw) ~=~\frac{1}{2}\left(F_1(\bw)+F_2(\bw)\right) ~=~ \frac{3\lambda}{2}\bw^\top\left(I+\frac{1}{2c\sqrt{d}}M\right)^{-1}\bw-\frac{\delta}{2}\be_j,
  \]

with an optimum at
  \[
  \bw^* = \frac{\delta}{6\lambda}\left(I+\frac{1}{2c\sqrt{d}}M\right)\be_j.
  \]

The following lemma establishes that the functions satisfy the strong
convexity, smoothness and relatedness requirements of the theorem. The proof
also establishes that the inverse in the definition of $F_1$ indeed exists.

\begin{lemma}\label{lem:valid}
$F_1$ and $F_2$ are $\lambda$ strongly-convex, $9\lambda$ smooth, and
$\delta$-related.
\end{lemma}
\begin{proof}
The Hessian of $F_2$ is $3\lambda I$, which implies that $F_2$ is $3\lambda$
smooth and strongly convex (and in particular, $\lambda$-strongly convex). As
to $F_1$, note that since $\norm{M}\leq c\sqrt{d}$, then
\[
\norm{\frac{1}{2c\sqrt{d}}M}\leq \frac{1}{2},
\]
The fact that the spectral radius and spectral norm of symmetric matrices coincide implies that the eigenvalues of the matrix $I+\frac{1}{2c\sqrt{d}}M$ lie between $1-\frac{1}{2}=\frac{1}{2}$ and
$1+\frac{1}{2}=\frac{3}{2}$. Thus, all the eigenvalues are strictly
positive, hence the matrix is indeed invertible as in the definition of
$F_1$. Moreover, the eigenvalues of the inverse lie in
$\left[\frac{1}{3/2},\frac{1}{1/2}\right]=\left[\frac{2}{3},2\right]$, and
therefore those of
$3\lambda\left(\left(I+\frac{1}{2c\sqrt{d}}M\right)^{-1}-\frac{1}{2}I\right)$
lie in
$\left[3\lambda\left(\frac{2}{3}-\frac{1}{2}\right),3\lambda\left(2-\frac{1}{2}\right)\right]=\left[\frac{\lambda}{2},\frac{9\lambda}{2}\right]$.
Thus, the spectrum of the Hessian of $F_1$ lie in $\left[\lambda,9\lambda\right]$, which implies that $F_1$ is $\lambda$-strongly convex and $9\lambda$ smooth.

To show $\delta$-relatedness, the only non-trivial part is upper-bounding the norm of the
difference of the quadratic terms, which equals the following:
\begin{align}
  &\norm{3\lambda\left(\left(I+\frac{1}{2c\sqrt{d}}M\right)^{-1}-\frac{1}{2}I\right)
  -
  \frac{3\lambda}{2}I}\notag\\
  &=3\lambda\norm{\left(I+\frac{1}{2c\sqrt{d}}M\right)^{-1}-I}.\label{eq:valid1}
  \end{align}
Since $\norm{M}\leq c\sqrt{d}$, the eigenvalues of
$\left(I+\frac{1}{2c\sqrt{d}}M\right)^{-1}-I$ lie between
$\frac{1}{1+1/2}-1=-\frac{1}{3}$ and $\frac{1}{1-1/2}-1=1$, which implies that
\eqref{eq:valid1} can be upper bounded by $3\lambda\leq\delta$.
\end{proof}

The next lemma proves the second part of the theorem, namely an upper bound on the suboptimality of any local function optimizer.

\begin{lemma}\label{lem:localopt}
  For any $\hat{\bw}_j=\arg\min_{\bw\in\reals^d}F_j(\bw)$, it holds that
  $F(\hat{\bw}_j)-\min_{\bw\in\reals^d}F(\bw)\leq c\delta^2/\lambda$ for some numerical positive constant $c$.
\end{lemma}
\begin{proof}
  The optimum of any quadratic and strongly-convex function $\bw^\top A \bw+\mathbf{b}^\top \bw+c$ equals $\frac{1}{2}A^{-1}\mathbf{b}$.  Therefore, if $\bw^*$ is the optimizer of $F$, and we denote the parameters of $F$ and $F_j$ by $A,\mathbf{b},c$ and $A_j,\mathbf{b}_j,c_j$ respectively, then
  \begin{align*}
    \norm{\hat{\bw}_j-\bw^*} &= \frac{1}{2}\norm{A_j^{-1}\mathbf{b}_j-A^{-1}\mathbf{b}}\\
    &= \frac{1}{2}\norm{A_j^{-1}\mathbf{b}_j-A^{-1}\mathbf{b}_j+A^{-1}\mathbf{b}_j-A^{-1}\mathbf{b}}\\
    &\leq
    \frac{1}{2}\left(\norm{\left(A_j^{-1}-A^{-1}\right)\mathbf{b}_j}+\norm{A^{-1}\left(\mathbf{b}_j-\mathbf{b}\right)}\right)\\
    &\leq
    \frac{1}{2}\left(\norm{A_j^{-1}-A^{-1}}\norm{\mathbf{b}_j}+\norm{A^{-1}}\norm{\mathbf{b}_j-\mathbf{b}}\right).
  \end{align*}
  By definition of $F_1,F_2$ and the average function $F$, this is at most      \begin{equation}\label{eq:vvv1}
  \frac{1}{2}\left(\norm{A_j^{-1}-A^{-1}}\delta+\norm{A^{-1}}\frac{\delta}{2}\right).
  \end{equation}
  In \lemref{lem:valid}, we showed that $F_1,F_2$ are $\lambda$-strongly convex and $9\lambda$ smooth, which implies that the eigenvalues of $A_j$ as well as $A$ lie in $\left[\frac{\lambda}{2},\frac{9\lambda}{2}\right]$. Therefore, the eigenvalues of $A_j^{-1}$ and $A^{-1}$ lie in $\left[\frac{2}{9\lambda},\frac{2}{\lambda}\right]$, so $\norm{A^{-1}}\leq \frac{2}{\lambda}$ and $\norm{A_j^{-1}-A^{-1}}\leq \frac{2}{\lambda}$.
	Substituting this back into \eqref{eq:vvv1}, we get
  \[
  \norm{\hat{\bw}_j-\bw^*} \leq \frac{1}{\lambda}\left(\delta+\frac{\delta}{2}\right) = \frac{3\delta}{2\lambda}.
  \]
  Finally, since $F$ is $9\lambda$-smooth, and its minimizer is $\bw^*$,
  \[
  F(\hat{\bw}_j)-F(\bw^*) \leq \frac{9\lambda}{2}\norm{\hat{\bw_j}-\bw^*}^2 \leq \frac{9\lambda}{2}\left(\frac{3\delta}{2\lambda}\right)^2,
  \]
  which equals $81\delta^2/8\lambda$ as required.
\end{proof}

We now turn to derive the lower bound in the theorem statement. As discussed earlier, the intuition is that the optimal point $\bw^*$ is a function of the $j$-th column of $M$, so the machines holding $F_1$ must broadcast enough information on $M$ to the designated machine producing the algorithm's output (the machine, by construction, holds $F_2$, and hence knows $j$ but not $M$). As long as the communication budget is smaller than the size of $M$, this will be difficult to achieve. This intuition is formalized in the following lemma, which is based on information-theoretic tools:

\begin{lemma}\label{lem:comm}
  For any dimension $d\geq c$ (where $c$ is the same constant as in \lemref{lem:tao} and the definition of $F_1$), and for any (possibly randomized) 1-round algorithm using at most
  $d^2/128$ bits of communication, there exists a valid choice of $M,j$ for the functions $F_1,F_2$ defined above,
  such that the vector $\hat{\bw}$ returned by the algorithm satisfies
  \[
  \E\left[\norm{\hat{\bw}-\bw^*}^2\right] \geq c'\left(\frac{\delta}{\lambda}\right)^2,
  \]
  where the expectation is over the algorithm's randomness, and $c'$ is a positive
  numerical constant.
\end{lemma}
Using the lemma and the $\lambda$-strong convexity of $F_1,F_2$ (and hence
their average $F$),
\[
\E[F(\hat{\bw})-F(\bw^*)] ~\geq~ \frac{\lambda}{2}~\E[\norm{\bw-\bw^*}^2] ~\geq~ \frac{c'}{2}\frac{\delta^2}{\lambda},
\]
hence proving the theorem.

It now remains to prove \lemref{lem:comm}:
\begin{proof}[Proof of \lemref{lem:comm}]
  By definition of $\bw^*$, we have that the $j$-th column of $M$, designated as $M_j$, satisfies
  \[
  M_j = 2c\sqrt{d}\left(\frac{6\lambda}{\delta}\bw^*-\be_j\right).
  \]
  Given the predictor $\hat{\bw}$ returned by the algorithm, define
  \[
  \hat{M}_j = 2c\sqrt{d}\left(\frac{6\lambda}{\delta}\hat{\bw}-\be_j\right).
  \]
  This can be thought of as the algorithm's `estimate' of the $j$-th column of $M$,
  based on the returned predictor.

  Define $[w]=\min\{1,\max\{-1,w\}\}$ as the clipping operation of a scalar $w$
  to $[-1,+1]$, and for a vector $\bw=(w_1,\ldots,w_d)$, define
  $[\bw]=([w_1],[w_2],\ldots,[w_d])$. By the expressions for $M_j,\hat{M}_j$ above, we have
  \begin{align*}
  \norm{[\hat{M_j}-M_j]}^2 &=~ \left\|\left[2c\sqrt{d}\frac{6\lambda}{\delta}\left(\hat{\bw}-\bw^*\right)\right]\right\|^2
  ~=~ \sum_{i=1}^{d}\left[\frac{12c\lambda\sqrt{d}}{\delta}\left(\hat{w}_i-w^*_i\right)\right]^2\\
  &\leq \left(\frac{12c\lambda\sqrt{d}}{\delta}\right)^2\sum_{i=1}^{d}(\hat{w}_i-w^*_i)^2,
  \end{align*}
  which implies that
  \begin{equation}\label{eq:commlem0}
  \norm{\hat{\bw}-\bw^*}^2 \geq \left(\frac{\delta}{12c\lambda\sqrt{d}}\right)^2\norm{[\hat{M_j}-M_j]}^2,
  \end{equation}
  To get the lemma statement, it
  is enough to show that for some $M,j$, one can lower bound $\E\left[\norm{[\hat{M_j}-M_j]}^2\right]$ (where the expectation is over the algorithm's randomness) by some constant multiple of $d$.

  Below, we will prove that if $M$ (in the definition of $F_1$) is chosen uniformly at random from all $\{-1,+1\}$-valued $d\times d$ symmetric matrices, and $j$ (in the definition of $F_2$) is chosen uniformly at random from $\{\lceil d/2 \rceil,\ldots,d\}$, then for any deterministic algorithm,
  \begin{equation}\label{eq:commlem1}
  \E_{M,j}\left[\norm{[\hat{M_j}-M_j]}^2\right]\geq \frac{d}{8}
  \end{equation}
  Let us first show how this can be used to prove the lemma. To do so, we
  will need the following lemma on the concentration of the spectral norm of
  random symmetric matrices.
  \begin{lemma}[\cite{terence2012topics}, Corollary 2.3.6]\label{lem:tao}
    There exist positive numerical constants $c,c'$, such that if $M$ is a
    $d\times d$ symmetric matrix, where each entry $M_{j,i},j\geq i$ is chosen independently and uniformly from
    $\{-1,+1\}$, and $d\geq c$, then $\Pr(\norm{M}> c\sqrt{d})\leq
    c\exp(-c' d)$.
  \end{lemma}
  First, we note that the expectation in \eqref{eq:commlem1} is over all symmetric $\{-1,+1\}$-valued matrices, including those
  whose spectral norm may be larger than $c\sqrt{d}$. However, by \lemref{lem:tao},
  $\Pr(\norm{M}>c\sqrt{d})\leq c\exp(-c' d)$ for some absolute constant
  $c'$. Letting $E$ be the event that $\norm{M}>c\sqrt{d}$, and noting that $\norm{[\bw]}^2\leq d$
  for any vector $\bw$, we have
  \begin{align*}
    \E\left[\norm{[\hat{M_j}-M_j]}^2\right] &=
    \E\left[\norm{[\hat{M_j}-M_j]}^2\middle| E\right]\Pr(E)+\E\left[\norm{[\hat{M_j}-M_j]}^2\middle|\neg E\right]\Pr(\neg E)\\
    &\leq d\Pr(E)+\E\left[\norm{[\hat{M_j}-M_j]}^2\middle|\neg E\right]\\
    &\leq cd\exp(-c' d)+\E\left[\norm{[\hat{M_j}-M_j]}^2\middle|\neg E\right].
  \end{align*}
  Plugging back into \eqref{eq:commlem1}, we get that
  \[
  \E\left[\norm{[\hat{M_j}-M_j]}^2\middle|\neg E\right] \geq \frac{d}{8}-cd\exp(-c_1 d),
  \]
  which is at least $d/16$ for any $d$ larger than some constant. Combining with \eqref{eq:commlem0}, we get
  \[
  \E\left[\norm{\hat{\bw}-\bw^*}^2\middle| \neg E \right] \geq \frac{1}{16}\left(\frac{\delta}{12c\lambda}\right)^2.
  \]
  This inequality implies that for any deterministic algorithm, in expectation over the random draw
  of $j$ and a $\{-1,+1\}$-valued matrix $M$ with spectral norm at most
  $c\sqrt{d}$, $\norm{\hat{\bw}-\bw^*}^2$ will be at least $c'\left(\frac{\delta}{\lambda}\right)^2$
  for some suitable constant $c'$. By Yao's minimax principle,
  this implies that for any (possibly randomized) algorithm, there will be some deterministic choice of $M,j$ such that $\norm{M}\leq c\sqrt{d}$, and
  for which
  \[
  \E\left[\norm{\hat{\bw}-\bw^*}^2\right]\geq
  c'\left(\frac{\delta}{\lambda}\right)^2
  \]
  (in expectation over the algorithm's randomness), yielding the lemma's statement.

  It now remains to prove \eqref{eq:commlem1}, assuming $j$ is chosen uniformly at random from $\{\lceil d/2 \rceil,\ldots,d\}$,
  and $M$ is chosen at random (i.e. each entry at or above the main diagonal is chosen independently and uniformly
  from $\{-1,+1\}$). Roughly speaking, the proof idea is to reduce this to
  an upper bound on how much information the machines holding $M$ can send on $M$'s entries (and more particularly, on the entries in
  the upper-right quadrant of
  $M$). Since this quadrant is composed of
  $\Theta(d^2)$ random variables, and the machines can send much less than
  $d^2$ bits, this information is necessarily restricted.
  
  Let $\Pr(\cdot)$ denote probability with
  respect to the random choice of $M,j$, and let ${\Pr}_{j}(\cdot)$ denote probability
  conditioned on the choice of $j$. Recalling that any entry $M_{j,i}$ in the $j$-th column has values in $\{-1,+1\}$, it follows
  that either $M_{j,i}$ has the same sign as $\hat{M}_{j,i}$, or that $([M_{j,i}-\hat{M}_{j,i}])^2$ is at least $1$.
  Therefore, we have the following:
  \begin{align}
    \E&\left[\norm{[M_j-\hat{M}_j]}^2\right] = \sum_{i=1}^{d}\E[([M_{j,i}-\hat{M}_{j,i}])^2]~\geq~
    \sum_{i=1}^{\lceil d/2 \rceil}\E[([M_{j,i}-\hat{M}_{j,i}])^2]\notag\\
    &\geq \sum_{i=1}^{\lceil d/2 \rceil}\E\left[([M_{j,i}-\hat{M}_{j,i}])^2\middle| M_{j,i}\hat{M}_{j,i}\leq 0\right]\Pr\left(M_{j,i}\hat{M}_{j,i}\leq 0\right)+0\notag\\
    &\geq \sum_{i=1}^{\lceil d/2 \rceil}\Pr\left(M_{j,i}\hat{M}_{j,i}\leq 0\right)~=~ \sum_{i=1}^{\lceil d/2 \rceil}\left(\frac{1}{1+\lfloor d/2 \rfloor}\sum_{j=\lceil d/2 \rceil}^{d}{\Pr}_{j}\left(M_{j,i}\hat{M}_{j,i}\leq 0\right)\right)\notag\\
    &=\frac{1}{1+\lfloor d/2 \rfloor}\sum_{i=1}^{\lceil d/2 \rceil}\sum_{j=\lceil d/2 \rceil}^{d}\left(\frac{1}{2}{\Pr}_{j}\left(\hat{M}_{j,i}\leq 0 | M_{j,i}>0\right)+
    \frac{1}{2}{\Pr}_{j}\left(\hat{M}_{j,i}\geq 0 | M_{j,i}<0\right)\right)\notag\\
    &\geq
    \frac{1/2}{1+\lfloor d/2 \rfloor}\sum_{i=1}^{\lceil d/2 \rceil}\sum_{j=\lceil d/2 \rceil}^{d}
    \left(1-\left({\Pr}_{j}\left(\hat{M}_{j,i}\geq 0 | M_{j,i}>0\right)-
    {\Pr}_{j}\left(\hat{M}_{j,i}\geq 0 | M_{j,i}<0\right)\right)\right)\notag\\
    &\geq \frac{1/2}{1+\lfloor d/2 \rfloor}\sum_{i=1}^{\lceil d/2 \rceil}\sum_{j=\lceil d/2 \rceil}^{d}\left(1-\left|{\Pr}_{j}\left(\hat{M}_{j,i}\geq 0 | M_{j,i}<0\right)-
    {\Pr}_{j}\left(\hat{M}_{j,i}\geq 0 | M_{j,i}>0\right)\right|\right)\notag\\
    &= \frac{\lceil d/2\rceil}{2}-\frac{1/2}{1+\lfloor d/2\rfloor}\sum_{i=1}^{\lceil d/2 \rceil}\sum_{j=\lceil d/2 \rceil}^{d}\left|{\Pr}_{j}\left(\hat{M}_{j,i}\geq 0 | M_{j,i}<0\right)-
    {\Pr}_{j}\left(\hat{M}_{j,i}\geq 0 | M_{j,i}>0\right)\right|.\label{eq:commlem2}
  \end{align}
  Let $S$ be the vector of bits broadcasted by the machines holding $F_1$, and received by the machine designated with providing the output (recalling that it only holds $F_2$). Note that
  conditioned on $S$ and $j$, the algorithm's output (and hence $\hat{M}_{j,i}$) is
  independent of $M$. Therefore, we have
  \begin{align*}
    &\left|{\Pr}_{j}\left(\hat{M}_{j,i}\geq 0 | M_{j,i}<0\right)-
    {\Pr}_{j}\left(\hat{M}_{j,i}\geq 0 | M_{j,i}>0\right)\right|\\
    &=\left|\sum_{S}{\Pr}_{j}\left(\hat{M}_{j,i}\geq 0 | S,M_{j,i}<0\right)\Pr(S|M_{j,i}<0)-
    \sum_{S}{\Pr}_{j}\left(\hat{M}_{j,i}\geq 0 | S,M_{j,i}>0\right)\Pr(S|M_{j,i}>0)\right|\\
    &=\left|\sum_{S}{\Pr}_{j}\left(\hat{M}_{j,i}\geq 0 | S\right)\Pr(S|M_{j,i}<0)-
    \sum_{S}{\Pr}_{j}\left(\hat{M}_{j,i}\geq 0 | S\right)\Pr(S|M_{j,i}>0)\right|\\
    &\leq \sum_{S}\left|{\Pr}_{j}\left(\hat{M}_{j,i}\geq 0 | S\right)\left(\Pr(S|M_{j,i}<0)-
    \Pr(S|M_{j,i}>0)\right)\right|\\
    &\leq \sum_{S}\left|{\Pr}_{j}(S|M_{j,i}<0)-
    {\Pr}_{j}(S|M_{j,i}>0)\right|\\
    &\leq \sum_{S}\left|{\Pr}_{j}(S|M_{j,i}<0)-{\Pr}_{j}(S)\right|+
    \sum_{S}\left|{\Pr}_{j}(S|M_{j,i}>0)-{\Pr}_{j}(S)\right|.
  \end{align*}
  Since $S$ is sent by the machines holding $F_1$ (and not $F_2$), it is independent of $j$. Therefore, we can write the above as
  \[
  \sum_{S}\left|{\Pr}(S|M_{j,i}<0)-{\Pr}(S)\right|+
    \sum_{S}\left|{\Pr}(S|M_{j,i}>0)-{\Pr}(S)\right|
  \]
  where $j$ in the conditioning is a fixed index.
  Using Pinsker's inequality, we can upper bound the above by
  \[
  \sqrt{2 D_{kl}\left(p(S|M_{j,i}<0)||p(S)\right)}+
  \sqrt{2 D_{kl}\left(p(S|M_{j,i}>0)||p(S)\right)}
  \]
  where $p$ is the probability distribution of $S$,
  and $D_{kl}$ is the Kullback-Leibler divergence. By the elementary inequality $\sqrt{a}+\sqrt{b}\leq
  \sqrt{2(a+b)}$ for all non-negative $a,b$, we can upper bound the above by
  \begin{align*}
  &\sqrt{4\left(D_{kl}\left(p(S|M_{j,i}<0)||p(S)\right)+
  D_{kl}\left(p(S|M_{j,i}>0)||p(S)\right)\right)}\\
  &=\sqrt{8}\sqrt{\frac{1}{2}\left(D_{kl}\left(p(S|M_{j,i}<0)||p(S)\right)+
  D_{kl}\left(p(S|M_{j,i}>0)||p(S)\right)\right)}.
  \end{align*}
  Using the fact that $M_{j,i}$ (for some fixed $j,i$) is uniformly distributed in $\{-1,+1\}$, and
  that the mutual information $I(X;Y)$ between random variables $X,Y$ equals
  $\E_{Y}\left[D_{kl}(p(X|Y=y)||p(X))\right]$, the above equals
  \[
  \sqrt{8}\sqrt{I(S;M_{j,i})}.
  \]
  Recalling that this is an upper bound on $\left|{\Pr}_{j}\left(\hat{M}_{j,i}\geq 0 | M_{j,i}<0\right)-
    {\Pr}_{j}\left(\hat{M}_{j,i}\geq 0 | M_{j,i}>0\right)\right|$, we have
  \begin{align}
    &\frac{1/2}{1+\lfloor d/2\rfloor}\sum_{i=1}^{\lceil d/2 \rceil}\sum_{j=\lceil d/2 \rceil}^{d}\left|{\Pr}_{j}\left(\hat{M}_{j,i}\geq 0 | M_{j,i}<0\right)-
    {\Pr}_{j}\left(\hat{M}_{j,i}\geq 0 | M_{j,i}>0\right)\right|\notag\\
    &~~\leq \frac{\sqrt{2}}{1+\lfloor d/2 \rfloor}\sum_{i=1}^{\lceil d/2 \rceil}\sum_{j=\lceil d/2 \rceil}^{d}\sqrt{I(S;M_{j,i})}~=~ \sqrt{2}\lceil d/2 \rceil\frac{1}{\lceil d/2\rceil \left(1+\lfloor d/2 \rfloor\right)}\sum_{i=1}^{\lceil d/2 \rceil}\sum_{j=\lceil d/2 \rceil}^{d}\sqrt{I(S;M_{j,i})}\notag\\
    &~~\leq \sqrt{2}\lceil d/2 \rceil\sqrt{\frac{1}{\lceil d/2\rceil \left(1+\lfloor d/2 \rfloor\right)}\sum_{i=1}^{\lceil d/2 \rceil}\sum_{j=\lceil d/2 \rceil}^{d} I(S;M_{j,i})} \label{eq:commlem3}
  \end{align}
  where the last step is by Jensen's inequality (i.e. the average of square roots is upper bounded by the square root of the average). The expression in the
  square root equals the average mutual information between a random variable $S$
  (composed of at most $d^2/128$ bits), and $\lceil d/2\rceil \left(1+\lfloor d/2 \rfloor\right)$ binary random variables $M_{j,i}$, where $i\in \{1,\ldots,\lceil d/2\rceil\},j\in \{\lceil d/2 \rceil,\ldots,d\}$, which are all independent by construction. By
  Lemma 6 in \cite{shamir2013fundamental}, it is at most $(d^2/128)/\left(\lceil d/2\rceil \left(1+\lfloor d/2 \rfloor\right)\right) \leq 1/32$, so we have
  \[
   \sqrt{2}\lceil d/2 \rceil\sqrt{\frac{1}{\lceil d/2\rceil \left(1+\lfloor d/2 \rfloor\right)}\sum_{i=1}^{\lceil d/2 \rceil}\sum_{j=\lceil d/2 \rceil}^{d} I(S;M_{j,i})}\leq \sqrt{2}\lceil d/2\rceil\sqrt{\frac{1}{32}} = \frac{\lceil d/2 \rceil}{4}.
  \]
  Recalling this is an upper bound on \eqref{eq:commlem3},
  which is the second term in \eqref{eq:commlem2}, we get that
  \[
  \E_{M,j}\left[[M_j-\hat{M}_j]^2\right] \geq \frac{\lceil d/2 \rceil}{2}-\frac{\lceil d/2\rceil}{4} = \frac{\lceil d/2 \rceil}{4}\geq \frac{d}{8},
  \]
  hence justifying \eqref{eq:commlem1}.
\end{proof}

\end{document}